\documentclass[12pt]{colt2022}

\def\notes{1}

\newif\ifCOLT
\COLTtrue %

\usepackage{amsmath}
\usepackage{amssymb} 

\usepackage{algorithm}
    \usepackage{algpseudocode}
    \usepackage{etoolbox}
\usepackage[normalem]{ulem}
\usepackage{paralist}

\usepackage{fancyhdr}
\usepackage{graphicx}
\usepackage{wrapfig}
\usepackage{ifthen}
\usepackage{bbm}
\usepackage{comment}

\usepackage[T1]{fontenc} 

\usepackage{thmtools, thm-restate}

\ifCOLT
\hypersetup{
    hidelinks = true
    }
\else
    \usepackage{subcaption}
    \usepackage{amsthm}
    \usepackage[usenames,dvipsnames]{xcolor}
    \usepackage[numbers,square]{natbib}
    \usepackage{fullpage}
    \usepackage[pdfencoding=auto, psdextra]{hyperref}
\fi

\makeatletter
\newcommand*{\algrule}[1][\algorithmicindent]{%
  \makebox[#1][l]{%
    \hspace*{.2em}%
  }
}

\newcount\ALG@printindent@tempcnta
\def\ALG@printindent{%
    \ifnum \theALG@nested>0%
    \ifx\ALG@text\ALG@x@notext%
    \else 
    \unskip
    \ALG@printindent@tempcnta=1
    \loop
    \algrule[\csname ALG@ind@\the\ALG@printindent@tempcnta\endcsname]%
    \advance \ALG@printindent@tempcnta 1
    \ifnum \ALG@printindent@tempcnta<\numexpr\theALG@nested+1\relax
    \repeat
    \fi
    \fi
}
\patchcmd{\ALG@doentity}{\noindent\hskip\ALG@tlm}{\ALG@printindent}{}{\errmessage{failed to patch}}
\patchcmd{\ALG@doentity}{\item[]\nointerlineskip}{}{}{} %
\makeatother

\ifCOLT
    \newtheorem{claim}[theorem]{Claim}
    \newtheorem{open}{Open Problem}[section]
    \newtheorem{fact}[theorem]{Fact}
    \usepackage{mathtools}
\else
    \theoremstyle{plain}
    \newtheorem{theorem}{Theorem}[section]
    \newtheorem{claim}[theorem]{Claim}
    \newtheorem{proposition}[theorem]{Proposition}

    \newtheorem{corollary}[theorem]{Corollary}
    \newtheorem{lemma}[theorem]{Lemma}
    \newtheorem{fact}[theorem]{Fact}
    
    \theoremstyle{definition}
    \newtheorem{definition}{Definition}[section]

    \theoremstyle{remark}
    
\fi

\mathtoolsset{showonlyrefs}

\DeclareMathOperator*{\E}{\mathbb{E}}
\newcommand{\Ex}{\E}

\newcommand{\bit}[1]{{{\{0,1\}}^{#1}}}

\newcommand{\Var}[1]{\mathrm{Var}\left(#1\right)}

\newcommand{\mixture}{\cP_{\mathrm{mix}}}
\newcommand{\qdist}{\cD}

\newcommand{\TaskCore}{Task C}

\newcommand{\TaskSubpopTestFixed}{Task B'}    %
\newcommand{\TaskSubpopDist}{Task B}  %
\newcommand{\TaskStream}{Task A}

\newcommand{\maxfixed}{\rho}

\newcommand{\CI}{\mathrm{CI}}
\newcommand{\tCI}{\widetilde{\mathrm{CI}}}

\newcommand{\struct}{\text{\itshape ``S''}}
\newcommand{\unif}{\text{\itshape ``U''}}

\newcommand{\slice}[3]{#1[#2:#3]}

\newcommand{\defeq}{\stackrel{{\mbox{\tiny def}}}{=}}
\newcommand{\eps}{\varepsilon}

\newcommand{\cB}{\mathcal{B}}
\newcommand{\cD}{\mathcal{D}}
\newcommand{\cF}{\mathcal{F}}
\newcommand{\cH}{\mathcal{H}}

\newcommand{\cJ}{\mathcal{J}}

\newcommand{\cP}{\mathcal{P}}
\newcommand{\cQ}{\mathcal{Q}}

\newcommand{\cU}{\mathcal{U}}
\newcommand{\cX}{\mathcal{X}}
\newcommand{\cY}{\mathcal{Y}}
\newcommand{\cZ}{\mathcal{Z}}
\newcommand{\1}[1]{\mathbbm{1}_{\{#1\}}}

\newcommand{\N}{N}
\newcommand{\T}{T}

\ifnum\notes=1
\newcommand{\mynote}[3]{\marginpar{\tiny \sf \color{#1} {#2}: {#3}}}

\else 
\newcommand{\mynote}[3]{}
\fi

\newcommand{\err}{\text{\rm err}}

\newcommand{\paren}[1]{{\left( {#1} \right)}}
\newcommand{\bparen}[1]{{\big( {#1} \big)}}
\newcommand{\Bparen}[1]{{\Big( {#1} \Big)}}
\newcommand{\abs}[1]{{\left| {#1} \right|}}
\newcommand{\braces}[1]{{\left\{ {#1} \right\} }}
\ifCOLT
\else
    
\fi

\newcommand{\ip}[1]{{\left\langle {#1} \right\rangle}}

\newcommand{\simiid}{\sim_{\text{\rm iid}}}

\DeclarePairedDelimiterX{\infdivx}[2]{(}{)}{%
  #1\;\delimsize\|\;#2%
}
\newcommand{\KL}{\mathrm{KL}\infdivx}
\newcommand{\JSD}{\mathrm{JSD}\infdivx}

 \coltauthor{\Name{Gavin Brown}\thanks{Supported in part by NSF awards CCF-1763786 and CNS-2120667, a Sloan Foundation research award and gifts from Apple and Google.}
 \Email{grbrown@bu.edu}\\
  \Name{Mark Bun}\thanks{Supported in part by NSF awards CCF-1947889 and CNS-2046425.} \Email{mbun@bu.edu}\\
  \Name{Adam Smith}\footnotemark[1] \Email{ads22@bu.edu}\\
  \addr Department of Computer Science, Boston University, Boston, USA.}
\title[Strong Memory Lower Bounds]{Strong Memory Lower Bounds for Learning Natural Models
}
\usepackage{times}

\begin{document}

\maketitle

\begin{abstract} 
    We give lower bounds on the amount of memory required by one-pass streaming algorithms for solving several natural learning problems. In a setting where examples lie in $\bit{d}$ and the optimal classifier can be encoded using $\kappa$ bits, we show that algorithms which learn using a near-minimal number of examples, $\tilde O(\kappa)$, must use  $\tilde \Omega( d\kappa)$ bits of space. Our space bounds match the dimension of the ambient space of the problem's natural parametrization, even when it is quadratic in the size of examples and the final classifier. For instance, in the setting of $d$-sparse linear classifiers over  degree-2 polynomial features, for which $\kappa=\Theta(d\log d)$, our space lower bound is $\tilde\Omega(d^2)$. Our bounds degrade gracefully with the stream length $\N$, generally having the form $\tilde\Omega\paren{d\kappa \cdot \frac{\kappa}{\N}}$.

    Bounds of the form $\Omega(d\kappa)$ were known for learning parity and other problems defined over finite fields. Bounds that apply in a narrow range of sample sizes are also known for linear regression. Ours are the first such bounds for problems of the type commonly seen in recent learning applications that apply for a large range of input sizes. 
\end{abstract}

\section{Introduction}
The complex models that power much of machine learning's recent success are typically fit to large data sets using streaming algorithms that process examples one by one, updating a stored model as they go. Their performance is often limited by their memory footprint as much as it is by the complexity of their calculations \citep[see e.g.,][]{vaswani2017attention,brown2020language,ramesh2021zero}.

We give new lower bounds on the space required to solve natural learning problems in a streaming model, where each example can be processed only once. Consider a stream of data elements $Z=(Z_1,...,Z_{\N})$ drawn from a distribution $P$ on labeled examples in the set $\cZ = \cX \times \cY$, where  $\cX$ denotes a set of possible feature vectors and  $\cY$ a set of possible labels (e.g., $\{0,1\}$). 

A learning algorithm's goal, given one pass over the stream $Z$, is to find a hypothesis $h:\cX\to\cY$ with small error on unseen examples from $P$. We focus on  misclassification error,   \( \err_P(h) \defeq \Pr_{(x,y)\sim P} (h(x) \neq y) \, ,\)
though for real valued functions (e.g., when $\cY = [0,1]$) we consider the expected absolute error $ \err_P(h) \defeq \Ex_{(x,y)\sim P} \abs{h(x) - y} \, .$
The error of a learning algorithm $A$ on $P$ with input length $\N$ is defined as the expected error on a stream of $\N$ inputs drawn from $P$. 
\[ \err_{P^{\otimes \N}}(A) \defeq \Ex_{Z\sim P^{\otimes \N}} \bparen{\err_P(A(Z))} \, .\]  We simply write $\err_{P}(A)$ when $\N$ is clear from context.

Given a function class $\cH$ of functions from $\cX\to \cY$,
we say that algorithm $A$ \textit{(agnostically) learns  $\cH$ to error $\eps$} if, for every distribution $P$ on $\cX\times \cY$, the algorithm $A$ finds a hypothesis with (expected) error at most $\eps$ more than that of the best hypothesis in $\cH$, that is,
 $\err_P(A) - \inf_{h \in \cH}\Bparen{\err_P(h)} \leq \eps \, .$
For concreteness, we will consider an algorithm successful it if learns a particular class to error better by a constant than one gets by random guessing. For example, we might assume there is an $h^* \in \cH$ for which $\err_P(h^*)\leq 1/4$ and require our algorithm have  expected error at most $0.49$. 

In this paper, we ask how much memory is needed for a streaming algorithm to agnostically learn $\cH$ as a function of the size $\kappa = \log |\cH|$ of the function space,\footnote{We focus on the cardinality of the function space for simplicity but, for continuous spaces, one should think of $\kappa$ as the bit length of an appropriate discrete representation of functions in $\cH$.} the data dimension $d = \log|\cX|$ and the stream length $\N$. A streaming  algorithm receives each example once;  the  algorithm's memory size on a given execution is the maximum number of bits used to encode its state between processing examples. The output is a function of the final state. 

If space is not a concern, observe that, for any class $\cH$ and any  distribution $P$, it suffices to receive $O(\kappa)$ examples from $P$ in order to find a hypothesis with error within a small constant of the best hypothesis in $\cH$. Thus, there is an algorithm that uses a stream of length $\N=O(\kappa)$ and space $O(\kappa d)$ bits which simply stores its entire stream and attains low expected error.

\paragraph{Strong Bounds for Sparse Models}
We show simple, natural function classes (capturing, for example, sparse linear classification over a degree-2 polynomial feature space) for which this trivial memory bound is tight. Specifically, every algorithm that uses a stream of length $\tilde O(\kappa)$ and finds a hypothesis whose accuracy exceeds random guessing by a constant must, on average over executions, use $\tilde \Omega( d \kappa) $ bits of memory. Our lower bounds do not assume any particular form or running time of the algorithm or its output; they build on information complexity techniques developed by \citet{braverman2020coin} for bounding the space complexity of statistical estimation.

In a breakthrough result, \citet{raz2018fast} proved such a bound for the class of parity functions over $d$-bit inputs. In his setting,  the function class $\cH$ is also the set of $d$-bit strings, so $d=\kappa$.
\citet{raz2018fast} proved that any streaming algorithm solving parity learning requires either $\Omega(d^2) = \Omega(d\kappa)$ bits of memory or $2^{\Omega(\kappa)}$ examples.
Subsequent papers extended and generalized these results, in particular to higher-degree polynomials and related classes (see Section~\ref{sec:related_work}).

These results are striking. However, they do not obviously imply lower bounds for the function classes (linear models, neural networks) that are the focus of much modern machine learning. 
Although several authors have studied memory bounds for problems of a more continuous flavor, initial lower bounds were generally limited to the form $\tilde \Omega(\max(\kappa , d))$—that is, the lower bounds are limited to either the size of the model or the size of a single example (\citet{SteinhardtD2015,garg2014communication,braverman2016communication}).
Several recent papers prove stronger lower bounds, but under significant restrictions on either parameter ranges \citep{sharan2019memory,dagan2019space,dagan2018detecting} or the computational model~\citep{marsden2022efficient}. We discuss these further in Related Work (Section~\ref{sec:related_work}). 

We consider a simple distributional problem, described below, and show a memory lower bound of $\tilde \Omega(\kappa d \cdot \frac{\kappa}{\N})$ for all $N\geq \kappa$. It implies memory lower bounds for learning a number of natural function classes. These include:

\newcommand{\kdictators}{Let $\cX=[k]\times \bit{d'}$ (for $d' = d - \log_2 k$, so inputs can be described with $d$ bits) and consider classifiers $h_{i_1,...,i_k}$ specified by $k$ indices in $[d']$, where $h_{i_1,...,i_k}(j,x)=x_{i_j}$ (that is, for each value $j$ there is a single bit of $x$ that determines the label).}
\newcommand{\sparsekernel}{Let $\cX = \bit{d}$, and consider classifiers of the form $h(x) = \text{sign}(\ip{w,\phi(x)})$ where $\phi(x)$ denotes the values of degree-2 monomials in the entries of $x$ (so each entry of $\phi(x)$ equals $x_ix_j$ for two indices $i,j \in [d]$) and $w \in \bit{\binom{d}{\le 2}}$ has at most $k$ nonzero entries.}
\newcommand{\ktermdnfs}{Let $\cX=\{0,1\}^d$ and consider functions given by the OR of $k$ terms, each of which is the AND of two input bits.}
\newcommand{\multiclasslinear}{Let $\cX = \bit{d}$ and $\cY = [k]$ (so there are $k$ distinct labels). Let $\cH$ comprise all functions of the form $h(x) =  \arg\max_{j \in [k]} \ip{w_j,x}$ where each $w_j \in \bit{d}$ has  $O(\log k)$ nonzero entries.} 
\newcommand{\realregression}{Let $\cX =\bit{d}$ and $\cY = [0,1]$. Consider functions realizable by a sparse two-layer neural network with a single hidden layer of $k$ ReLU nodes, each of which is connected to at most $O(\log k)$ input nodes. 
The weights on the wires in the first layer are either 0 or 1, and those in the second layer are in $\braces{0,\frac{1}{k-1},\frac{1}{k-1},..., 1}$.}
\begin{compactitem}
    \item 
        \textit{Direct sums of $k$ dictators:} 
        \kdictators
        
    \item 
        \textit{Sparse linear classifiers over  degree-2 polynomial features}: 
        \sparsekernel 
        ~The classifiers we study may also be viewed as \textit{$k$-term 2-DNFs}: 
        \ktermdnfs

    \item \textit{Multiclass sparse linear classifiers}:  
        \multiclasslinear
        ~The combining function can also be taken to be a ``softmax'' instead of the exact argmax.

    \item \textit{Real-valued regression}: 
        \realregression
\end{compactitem}

For each of these settings, $\kappa = \tilde\Theta(k \log d)$ and we show a space lower bound of $\tilde \Omega(d k)=\tilde \Omega(d \kappa)$  when $\N$ is close to the minimal sample complexity of $\Theta(\kappa)$.  For two of the classes above, sparse linear classifiers  over the polynomial features and $k$-term 2-DNFs, our bounds apply for $1\leq k\leq d/2$. Our bounds for the remaining problems apply for $k$ anywhere from 1 to superpolynomial in $d$.

Our bounds degrade gracefully as the sample size increases. For $\N \geq \kappa$, every learner that succeeds with a stream of $\N$ examples requires memory at least $\tilde \Omega( d \kappa^2 / \N) $. In this regard, our bounds behave similarly to the initial, weaker bounds for regression-like problems, but are not as strong as those for learning parity and other algebraic problems~(as in \citet{raz2018fast}).

For direct sums of $k$ dictators, our lower bounds are matched up to logarithmic factors in all parameter regimes by a simple, quasilinear time agnostic learning algorithm (Appendix~\ref{app:agnostic_dictators}). For the other function classes, there are matching algorithms for the particular input distributions that arise in our arguments but it is unclear if the same is true for general distributions.

\paragraph{A Simple, Distributional ``Core'' Problem} Our bounds all derive from space lower bounds for the following simple problem, described more completely in Section~\ref{sec:task_definitions} and parametrized by positive integers $d$, $k$, and $\maxfixed \leq d$. The learner receives a stream of $\N$ inputs in $[k]\times \bit{d}$, each of which consists of a ``subpopulation identifier'' in $[k]$ and a feature vector in $\bit{d}$.

\begin{itemize}
    \item The input stream is drawn i.i.d.~from a distribution $P$ which is a uniform mixture of $k$ components. Each component $j$'s distribution is specified by a set $I_j$ of up to $\maxfixed$ indices in $[d]$ and bits $\paren{b_{j,i}}_{i \in I_j}$. An observation from component $j$ is a pair $(j,X)$ where $X\in \bit{d}$ is uniform except for coordinates in $I_j$, which are set to their $b_{j,i}$ values. The entire distribution $P$ is thus specified by  $k$ sets $I_1,...,I_k$ and the associated bits $\paren{b_{j,i}}_{j\in [k], i \in I_j}$. 

\item 
To generate the parameters of $P$, each component's parameters are generated independently by sampling a number $r$ uniformly in $ \{0,\ldots, \maxfixed\}$, and then selecting a size-$r$ subset $I_j$ of fixed features, with the bit values of the features selected uniformly.

\item 
After receiving $\N$ examples drawn from $P$, the algorithm  is presented with a test pair $(j,x)$ which is either drawn from $P$ (the ``structured'' case) or drawn uniformly at random from $[k]\times\bit{d}$ (``uniform''); the algorithm must distinguish between these two cases.

\end{itemize}

One can reduce this distributional problem to agnostic learning of any of the classes mentioned above—see Appendix~\ref{agnostic2core}.

For a  distribution $P$ in the class above, there is a simple optimal distinguisher: given a pair $(j,x)$,  it checks if $x$ agrees with $b_{j,i}$ in each position $i \in I_j$; it outputs ``structured'' if all the checks pass and ``uniform''  otherwise. Even when $\maxfixed =1$ (so there are either 0 or 1 fixed bits in each subpopulation), this distinguisher has advantage  $1/8$ over random guessing. More generally, it has advantage $\frac 1 2-O\paren{\frac{1}{\maxfixed}}$. It is not hard to learn such a distinguisher: for all $\rho\geq 1$, a distinguisher with constant advantage over random guessing can be learned from $\Theta(k \log(d))$ examples using space $O(dk)$. 
More generally, for $\N = \Omega( k \log(d))$ and $1\leq \maxfixed \leq d$, a simple strategy (described for completeness in Appendix~\ref{app:core_ubd}) learns a distinguisher with constant advantage in space $O\paren{dk \cdot \frac{k}{\N} \cdot \frac{1}{\maxfixed}}$. 
We show that these simple strategies are essentially optimal.

\begin{theorem}[Informal, see Theorem~\ref{theorem:core_lbd}]
    Consider the above streaming problem with $\N$ examples, $d$ dimensions, $k$ components,  and at most $\maxfixed$ fixed features, with $\maxfixed = o(d^{1/4})$.
    Any algorithm solving this task to constant error less than $\frac{1}{2}$ requires space $\Omega\left(\frac{k^2 d}{\N \maxfixed^4}\right) = \Omega\paren{dk \cdot \frac{k}{\N} \cdot \frac{1}{\maxfixed^4}}$.
\end{theorem}
The ratio $\N/k$ is the expected number of examples from each subpopulation. For $\rho$ that is at most logarithmic in $d$, the bounds have the form $\tilde\Omega\paren{k \cdot \frac {d}{\T}}$, where $\T= \N/k$. The bound may be viewed as incorporating two statements: the memory required to learn to distinguish a particular subpopulation using $\T$ examples scales as $\frac{d}{\T}$, and there is no better way to solve the larger problem than to learn each subpopulation individually.

\paragraph{Discussion} 
A widespread strategy in modern deep learning is to first train a large, dense network and then use it to find a smaller network by distillation or pruning. One common explanation for this approach is that optimization in the larger space is easier (see, e.g., \citet{frankle2020pruning,BartlettMR2021survey}). Our work suggests a different explanation for this strategy's empirical success---namely, the larger parameter vector allows the training process to encode information whose relevance to the problem can only be understood later. In particular, we show that training algorithms for sparse models must sometimes use space proportional to the ambient dimension of the natural encoding,
rather than with the size of the examples or the final classifier.

Our lower bounds hold only for one-pass streaming algorithms. This covers the common training strategy in settings where large amounts of data are available  \citep[e.g.,][who use partial epochs for some corpora]{brown2020language}. However, when data is not so abundant, machine learning models are often trained by taking multiple passes over the data sets. We conjecture that similar bounds hold for multi-pass algorithms. That is, it seems likely that a stream of $a\cdot n$ fresh examples is at least as useful as $a$ passes over $n$ examples; however, proving statements of that nature is challenging (\citet{GargRT2019twopass} and \citet{dagan2018detecting} provide notable successful examples), and we leave it as an open problem for future work.
 
\paragraph{Techniques} We bound an algorithm's space usage using specific measures of information complexity. The core distributional problem we use is inspired by the clustering problem used by \cite{BrownBFST21memorization} to prove that high-accuracy learning algorithms must sometimes store considerable information about individual examples in their final hypothesis. We change the problem in a few respects (fewer fixed bit positions; a focus on distinguishing instead of labeling). More importantly, the current paper employs a very different technical approach.

Our ``single subpopulation'' task, defined in Section~\ref{sec:task_definitions}, is closely related to the ``hide-and-seek problem'' of \citet{shamir2014fundamental}.
Both problems consider streams where examples are $d$-bit strings.
In each example, all but a few (a priori unknown) indices are uniformly random.
Beyond some technical details in the setting, the relevant difference is the notion of information cost used in our lower bound which, as we discuss below, allows us to lift the results to larger problems.

Our main technical tool is a notion of information cost recently introduced in \citet{braverman2020coin}, which we sometimes refer to as the ``composable information cost'' of algorithm $M$, denoted $\CI(M)$:
\begin{align}
    \CI(M) \defeq \sum_{i=1}^{\N} \sum_{t=i}^{\N} I(M_t; X_i\mid M_{i-1}), \label{eq:CI_def}
\end{align}
where $X_i$ is the $i$-th example and $M_t$ is the algorithm's state after processing $X_t$. 
We make no assumptions on the form of the memory state.
$I(\cdot ; \cdot | \cdot)$ denotes Shannon's (conditional) mutual information.
This information cost is always measured relative to a specific distribution on the input $X$. 
Composable information cost is useful for streaming applications when proving direct-sum-type statements, turning a lower bound for a simpler problem into a bound for a more complex one.
Furthermore, $\frac{1}{\N}\cdot \CI(M)$ is a lower bound on the space used by the algorithm (Lemma~\ref{lemma:space_lbd_SF}).

A key first step, which illustrates the measure's utility, is  a new lower bound we prove in Section~\ref{sec:taskstream_lbd} for the task of distinguishing whether a stream of bits is uniformly random or fixed (either to zero or one)---a special case of the coin problem~\citep{braverman2020coin,BravermanGZ21tightcoin}.
In the coin problem, the learner is asked to distinguish whether a stream of flips resulted from a biased or unbiased coin.
We prove that, when $X$ is uniformly distributed,  $\CI(M) = \Omega(1)$. By itself, this bound only implies (trivial) space usage of at least one bit. 
However, the bound is nontrivial: it shows that the information cost \textit{is nonzero even conditioned on the answer} (in this case, ``uniform''). Together with structure of composable information, it allows us to derive lower bounds for our core task.

Our argument that lifts a memory bound for a 1-bit stream to $k$ interleaved $d$-bit streams is inspired by the arguments of \cite{braverman2020coin} (namely, lower bounds for the Simultaneous $k$-Coins Problem and the random-order $k$-Coins Problem). Our applications require extension and modification of these techniques, which we now describe.

In some of our proofs, it is easier to work with a subset of the terms in~\eqref{eq:CI_def}: we define $\tCI$ as just the ``diagonal terms'' of $\CI$:
\begin{align}
    \tCI(M) \defeq \sum_{t=1}^{\N}  I(M_t; X_t\mid M_{t-1}). \label{eq:tCI_def}
\end{align}
Since mutual information is nonnegative, we have $\CI(M)\ge \tCI(M)$.
Furthermore, $\tCI$ allows us to consider a ``full-memory'' version of $M$ that observes all previous states.
We discuss this more in Section~\ref{sec:taskstream_lbd}.

Our reductions use an algorithm for a ``big'' task to construct an algorithm for a ``small'' task.
As is standard in information complexity arguments, this involves generating synthetic inputs to feed to the bigger algorithm alongside the real input.
When these inputs are uniformly distributed, this introduces no overhead, but in our problems the inputs must be drawn from a specific distribution.
This synthetic distribution is associated with a set of parameters $F$, which is itself a random variable and cannot be hard-wired into the algorithm. 
Our approach is simplified by conditioning on these parameters in the information cost itself, as
\begin{align*}
    \CI(M\mid F) \defeq \sum_{i=1}^{\N} \sum_{t=i}^{\N} I(M_t; X_i\mid M_{i-1}, F).
\end{align*}
We define $\tCI(M\mid F)$ similarly.
This means that our information complexity bounds are really measuring how an algorithm must store information about its input that is \textit{not about the actual distribution}. In this sense, the bounds show that some form of memorization (in the spirit of \cite{BrownBFST21memorization}) is needed at intermediate points in the computation. 

Finally, we extend the treatment of random-order streams in \citet{braverman2020coin}.
They apply a lower bound for the (single-coin) Coin Problem to prove a lower bound for the $k$-Coins Problem, where at each time step the learner receives an update from a randomly chosen coin.
This theorem is then applied to data streaming problems (such as $\ell_2$ Heavy Hitters) in the random-order model.
This setting is similar to our core problem, and indeed our reduction is almost identical.
We depart from \citet{braverman2020coin} in the analysis: 
    they define a notion of a ``good'' sequence of arrivals
    and, for any fixed good sequence, prove a lower bound for algorithms operating on that sequence.
This suffices for their purposes
but not for ours. For our learning task, an algorithm with advance knowledge of the sequence of arrivals can get by with lower information cost.
Instead, in our analysis, we let the sequence $S \in [k]^\N$ be a random variable, and prove a lower bound on the expression $\CI(M\mid S, F)$.
Although we condition on $S$, we exploit the learner's uncertainty, at any given time, about the part of $S$ it has not yet seen.
This enables a more direct proof, which we believe may be useful in other random-order streaming applications. 

\section{Related Work}\label{sec:related_work}

\paragraph{Lower Bounds Following Raz's Argument} As mentioned in the introduction, one closely related line of work proves memory lower bounds for problems defined over finite fields, such as learning parities \citep{raz2018fast,garg2018extractor,GargRT2019twopass}. The closest classes to the ones we consider are sparse parities. However, the techniques in the literature appear limited to analyzing parities that depend on at least $\log(d)$ variables. Lower bounds for these do not obviously imply bounds for the continuous function classes normally used in practice.

The algebraic bounds were generalized to rely on combinatorial conditions such as two-source extraction~\citep{garg2018extractor}, mixing~\citep{moshkovitz2017mixing,beame2018time}, and SQ dimension~(\citet{moshkovitz2017general} and \citet{gonen2020towards}, following early work of \citet{SteinhardtVW16}). 
These frameworks do not appear to yield nontrivial bounds for the our parameter settings, since the function classes we consider do not obviously satisfy the required combinatorial properties (in particular, they are very poor extractors and have low SQ dimension).

\paragraph{Lower Bounds for Single-Sample Distributed Regression}The work on learning parity was inspired by a related line of work on learning in streaming and distributed models that focused on regression-like problems
\citep{SteinhardtD2015,braverman2016communication}. The most directly relevant works consider $k$-sparse regression, showing lower bounds of $\tilde \Omega(d)$ on the memory required to learn with the minimal number of samples, $\tilde\Theta(k)$ (and lower bounds of $\tilde \Omega\paren{d \cdot \frac{k}{\N}} $ in general). Lower bounds such as ours, which exceed the data dimension, do not fit into these papers' distributed framework, since the protocol in which each player broadcasts their input always succeeds at learning and uses $O(d)$ bits of communication per player. That said, our bounds for the case $k=1$ are essentially bounds on distributed regression. They require new proofs in order to bound an appropriately composable notion of information complexity.

\paragraph{Strong Lower Bounds for Restricted Parameter Ranges}
For linear regression in the streaming setting, \citet*{sharan2019memory} give a memory lower bound of $\Omega(d^2) = \Omega( d\kappa)$ as long as $\N=o(d \log \log (1/\eps))$, where $\eps$ is the accuracy parameter. This bound exceeds the size of a single example, but applies in a somewhat narrow range of sample complexity (or, seen as a lower bound on sample complexity, exceeds the unrestricted sample complexity of $\Theta(d)$ by a factor of $\log\log(1/\eps)$).
Independent work of \citet{dagan2019space} also gave a memory lower bound of $\Omega(d^2)$ for linear regression; although not stated as a distributional problem, it can be interpreted as applying to streams of length exactly $\N=d-1$.

\citet{dagan2018detecting} give memory-sample tradeoffs for statistical estimation tasks, building on prior work by \citet{shamir2014fundamental}.
Their key theorem gives lower bounds for distinguishing between certain families of distributions.
An example is the set of distributions over $\{-1,+1\}^d$ where a single pair of indices $i,j$ satisfies $\E[x_ix_j]=\eps$ and all other pairs satisfy $\E[x_{i'}x_{j'}]=0$. They prove that a streaming algorithm that solves this task with  $\N$ samples requires 
memory size $\tilde\Omega(d^2\cdot \frac{1/\epsilon^2}{\N})$ for very low error, $\eps=\tilde{O}(d^{-1/3})$. 
Their bound is matched (up to logarithmic factors) by upper bounds in the low-correlation regime they consider; it also degrades gracefully with the number of passes over the stream.
While their techniques can be extended to supervised learning of some hypothesis classes (such as 1-sparse linear predictors that use degree-2 monomials), they do not extend to the constant-error regime we consider.

\paragraph{Streaming Lower Bounds for Other Statistical Problems} 
The work whose techniques we use most directly is that of \citet*{braverman2020coin} (extended by \citet*{BravermanGZ21tightcoin}). They considered the following \textit{$k$-coin problem}: suppose there are $k$ different coins. The algorithm receives a stream of examples of the form $(j,b)$ where $b$ is the result of a fresh toss of coin $j$. Its goal is to determine whether all the coins are fair or if some coin has bias greater than $\beta \approx \frac{1}{2} + \frac{1}{\sqrt{n}}$. 
We adopt the composable measure of information cost developed in that paper, and its use is critical in obtaining our final bounds. 

\citet*{diakonikolas2019communication} give a lower bound on testing uniformity of a distribution and related problems. They show that a streaming algorithm which, given $\N$ samples drawn i.i.d.~from a distribution over $[2^d]$, reliably distinguishes the uniform distribution from a distribution that is $\Omega(1)$-far from uniform must use memory $\tilde\Omega\paren{\frac{2^d}{\N}} $ for $\N\geq 2^{d/2}$. Although the lower bound far exceeds the size of any one sample and of the (one-bit) output, we are not aware of how to use it to derive a nontrivial bound for our setting. 

Concurrent work of
\citet{marsden2022efficient} proved memory lower bounds for convex optimization algorithms that use first-order queries, showing that any algorithm using significantly less than $d^{1.25}$ bits of memory requires a polynomial factor more queries than a memory-unconstrained algorithm in the low-error regime (i.e., $\eps =o(d^{-6})$ for their lower bound to apply). 
The computational setting of first-order queries is significantly different from the general learning setting we consider; understanding the models' relationship is an interesting topic for future work.

\paragraph{Memorization.}
In a batch setting, \citet{BrownBFST21memorization} established learning tasks for which any high-accuracy learning algorithm, on receiving a data set $X$ in $(\bit{d})^\N$, must output a classifier $M$ that satisfies $I(M; X\mid F)\ge \Omega(Nd)$, where $F$ is the set of parameters of the data-generating distribution.
As mentioned before, we rely on a modification of their hard distribution, but our techniques and results are significantly different.
They lower bound the size of the learned hypothesis (this implies a lower bound on the space usage of the algorithm).
Our results apply directly to the space needed by the algorithm, even when the hypothesis is itself small.
The hard distribution in \citet{BrownBFST21memorization} is tailored to the sample size, while our lower bounds fix the distribution and apply to a wide range of stream lengths.
Finally, our lower bounds apply to any learner with a small constant advantage over random guessing instead of only learners with nearly Bayes-optimal accuracy. Identifying deeper connections between space lower bounds and memorization \citep{BrownBFST21memorization,bassily2018learners,livni2020limitation} remains a fascinating open question.

\section{Task Definitions}\label{sec:task_definitions}

We define the core distributional task we study, as well as the simpler intermediate problems we analyze in order to understand it. 
We define and analyze an additional problem, \TaskSubpopTestFixed, in Appendix~\ref{app:inputs_to_example}.
Here, and throughout the paper, we use the notation $\Delta(S)$ to refer to set of distributions over a set $S$ and the notation $\cU$ to refer to a uniform distribution (with the space implied by context).
We use $\unif$ and $\struct$ to denote the outputs referring to ``uniform'' and ``structured,'' respectively.

\begin{definition}[Meta-Distributions]\label{def:structure}
    Parameters: positive integers $k, d, \maxfixed\le d$.
    We sample a distribution $\cP\in \Delta(\bit{d})$ from the \emph{structured subpopulation meta-distribution} $\cQ_{d,\maxfixed}\in \Delta(\Delta(\bit{d}))$ as follows: 
    \begin{enumerate}
        \item Draw $r\in \{0,\ldots, \maxfixed\}$ uniformly.
        \item Draw uniformly $\cJ = \{j_1,\ldots j_r\}\subseteq [d]$ (no replacement) and $\cB =(b_1, \ldots, b_r) \in \{0,1\}^r$.
    \end{enumerate}
    To draw $x\in \{0,1\}^d$ from  $\cP = \cP_{(\cJ,\cB)}$, set $x_{j_i}\gets b_i$ for each $j_i\in \cJ$.
    For $j\notin \cJ$, set $x_j \in \{0,1\}$ uniformly and independently.\footnote{That is, $\cP$ is uniform over the $n-r$-dimensional hypercube specified by $\cJ$ and $\cB$.}

    We sample a distribution $\mixture\in \Delta([k]\times \bit{d})$ from the \emph{structured population meta-distribution} $\cQ_{k,d,\maxfixed}\in \Delta(\Delta([k]\times\bit{d}))$ as follows: 
    \begin{enumerate}
        \item For $j =1,...,k$, draw $\cP_j \sim \cQ_{d,\maxfixed}$ and let $F_j = (\cJ_j, \cB_j)$ denote its parameters. 
    \end{enumerate}
    Define $\mixture\in\Delta([k]\times\bit{d})$ as a uniform mixture of $k$ distributions of the form $\delta_{j} \otimes \cP_j$, with $\delta_{j}$ denoting the point mass on $j$.

    Let $F$ denote the pair $(\cJ,\cB)$, the parameters of $\cP$.
    Let $F_{\mathrm{mix}}=(F_1,...,F_k)$ denote the parameters of $\mixture$.
    We call both $\cP$ and $\mixture$ \emph{structured} distributions.
\end{definition}

\newcommand{\coreadvantage}{
    \begin{equation*}
        \E_{\substack{\mixture \sim \cQ_{k,d,\maxfixed} \\ 
                    x=(x_1,\ldots,x_\N) \simiid \mixture \\ 
                    m \gets M(x)}} 
                    \left[\Pr_{\substack{ (j,y)\sim \mixture}}[m(j,y)= \struct] - \Pr_{\substack{ (j,y)\sim \cU}}[m(j,y)= \struct] \right].
    \end{equation*}
}

\begin{definition}[Task Definitions]\label{def:alltasks}
    $\N, \T, k, d$ and $\maxfixed\le d$ are positive integer parameters.
    \begin{enumerate}
        \item (\TaskCore, Core Problem)
            Learning algorithm $M$ receives a stream of $\N$ i.i.d.~samples from $\mixture$.
            After $\N$ steps, the learner outputs a (possibly randomized) function $m: [k]\times\{0,1\}^d\to \{\unif,\struct\}$.
            The \emph{advantage} of the learner is
            \coreadvantage  %
        \item (\TaskSubpopDist, Single Subpopulation)
            Learner $M$ receives a stream of $\T$ examples from $\cP$, with each example in $\bit{d}$, and outputs a value in $\{\unif,\struct\}$.
            The \emph{advantage} of the learner is
            \begin{align*}
                \Pr_{\substack{\cP\sim \cQ_{d,\maxfixed} \\ 
                                x=(x_1,\ldots,x_\N) \simiid \cP}}
                                [M(x)=\struct] - \Pr_{\substack{x=(x_1,\ldots,x_\N)\simiid \cU}}[M(x)= \struct].
            \end{align*}
        \item (\TaskStream, One-Bit Stream)
            The learner $M$ receives a stream of $\T$ bits and outputs a value in $\{\unif,\struct\}$. 
            The \emph{advantage} of the learner is
            \begin{align*}
                \Pr_{\substack{b \sim \cU \\ x = (b,b,...,b)}}[M(x)=\struct] - \Pr_{\substack{x=(x_1,\ldots, x_\T)\simiid \cU}}[M(x)=\struct].
            \end{align*}
        \end{enumerate}
\end{definition}

\section{A Lower Bound for the One-Bit Stream Task}\label{sec:taskstream_lbd}

\newcommand{\onebitlbd}{
    Consider a streaming algorithm $M$ for \TaskStream~on $\T$ inputs that has advantage at least $\delta$.
    Let $M_1, \ldots, M_{\T}$ be the memory states of $M$ when run on uniformly random inputs $X_1,\ldots, X_{\T} \in\{0,1\}$.
    Then
    \begin{equation*}
             \tCI(M) \defeq \sum_{t=1}^\T I(M_t;X_t\mid M_{t-1}) \ge \frac{\delta^4}{40}.
    \end{equation*}
    }
    
\begin{theorem}
    \label{theorem:one_bit_lbd}
    \onebitlbd
\end{theorem}

Recall that, although by itself this lower bound only implies a trivial space lower bound of $\Omega(1)$ bits, it is nontrivial: we measure information about the stream \textit{when the stream is uniform}.
In other words, the algorithm must contain information about the stream, even with respect to an observer who knows the answer.
In later sections, we use this statement to prove lower bounds for the ``larger'' tasks.
The proof has a simple but subtle setup. 
We now outline our approach; the full proof is in Appendix~\ref{app:taskstream_lbd_proofs}.

By assumption, $M$ has advantage $\delta$.
Since the structured data distribution chooses to fix the stream to 0 or 1 with equal probability, there exists a value $b^*\in\{0,1\}$ such that 
\begin{align}
    \Pr[M(b^*\cdot 1^{\T})=\struct] - \Pr_{x_{\le \T}\sim \cU }[M(x_{\le \T})=\struct] \ge \delta,
    \label{eq:one_bit_advantage}
\end{align}
where $1^{\T}$ denotes the length-$\T$ vector of ones.
For simplicity we assume $b^*=1$; the calculations do not rely on this choice.
Consider a stream generated from the following distribution, distinct from the structured distribution
The distribution over streams is a mixture of two distributions, with the first mixture component being uniform over all streams (corresponding to i.i.d.~bits).
The second mixture component places all of its mass on the stream with every entry equal to $b^*=1$.
Let $E$ denote the event that this stream is drawn from the second component, and $\bar E$ the event that it is drawn from the first.

We now consider an observer that, at time $t$, sees $M_{\le t}$, all the previous states of $M$.\footnote{As we make precise below, Lemma~\ref{lemma:condition_on_all_states} states that moving to an observer with access to all previous states (instead of just the current state) does not affect our argument.}
At each time $t$, the observer computes a posterior belief about event $E$:
\begin{align}
    p_t \defeq \Pr[E\mid M_{\le t}=m_{\le t}],
\end{align}
with prior $p_0=\frac{1}{2}$.
We analyze the random process $P_0,P_1,\ldots,P_T$ formed by these posteriors when the inputs $X_1,\ldots,X_T$ are uniformly random.
This process depends on the inputs and any random choices of the algorithm.
By construction the sequence begins at $\frac{1}{2}$ and, as we prove in Appendix~\ref{app:taskstream_lbd_proofs}, on average decreases by about $\delta^2$ over $\T$ time steps.
\newcommand{\finalposteriorgap}{
    Assume algorithm $M$ has advantage at least $\delta$.
    Let random variable $P_T$ be the final posterior of $M'$, i.e., $P_T=\Pr[E\mid M_{\le T}]$.
    When the inputs are uniform, we have
        $\E[P_T] \le \frac{1}{2} - \frac{\delta^2}{4}$.
}
\begin{lemma}\label{lemma:final_posterior_gap}
    \finalposteriorgap  
\end{lemma}
We will looks at the ``gaps'' of the form $P_t - P_{t-1}$, the difference between consecutive time steps.
Lemma~\ref{lemma:final_posterior_gap} shows that the random process, on average, must have a few time steps with large gaps or many time steps with moderate gaps.
We will show how these gaps reveal information about the input bits.

We restate the proof setup for emphasis.
The algorithm $M$ has $\delta$ advantage in the setup of Definition~\ref{def:alltasks}, where the stream may be fixed to 1's, fixed to 0's, or uniform.
The observer is told that it sees $M$ run on either the stream fixed to $b^*=1$ or the uniform stream; it calculates a posterior belief about which case is true.
In this proof, we analyze the sequence of these posteriors when the stream is actually uniform.

\newcommand{\qone}{\ensuremath{\qdist_{1,  m_{<t}}}}
\newcommand{\qzero}{\ensuremath{\qdist_{0, m_{<t}}}}
\newcommand{\mone}{\ensuremath{\mu_{1, m_{<t}}}}
\newcommand{\mzero}{\ensuremath{\mu_{0, m_{<t}}}}
Let us zoom in on a single time step $t$.
Fix the previous memory states $m_{<t}$: this also fixes the posterior $p_{t-1}$ and the distribution over the next posterior $P_t$.
This distribution is a uniform mixture of two components: let $\qone$ be the distribution over $P_t$ conditioned on $X_t=1$, and let $\qzero$ be the distribution conditioned on $X_t=0$.
When $X_t$ is uniform, $P_t$ is thus drawn from the mixture $\frac{1}{2}\paren{\qone+\qzero}$.
Denote the means of these distributions $\mone$ and $\mzero$.
We have $\E[P_t\mid m_{<t}] = \frac{1}{2}\paren{\mone+\mzero}$ and, furthermore, we prove that $p_{t-1} \le \mone$.
This yields the following lemma, connecting the gaps in the posterior to the difference in the means of the mixture components. 
\newcommand{\posteriorgapmeans}{
    For some time $t$, fix the memory states $m_{<t}$.
    Let $\mone = \E[P_t\mid X_t=1, M_{<t}=m_{<t}]$ and $\mzero = \E[P_t\mid X_t=0, M_{<t}=m_{<t}]$.
    We have
    \begin{align}
        \E[P_t - P_{t-1} \mid M_{<t}=m_{<t}]  \ge  \frac{\mzero - \mone}{2}.
    \end{align}
}
\begin{lemma}\label{lemma:posterior_gap_means}
    \posteriorgapmeans
\end{lemma}
Both sides of the inequality are negative.
When the distributions $\qone$ and $\qzero$ are distinct, the posterior $P_t$ will contain some information about the input $X_t$;
this lemma gives us a foothold to connect the gaps in the posterior with the difference in distributions.

We now move to working directly with the information contained in the posterior.
Lemma~\ref{lemma:condition_on_all_states} allows us to, in the expression for $\tCI$, condition on all previous memory states.
Since $P_{t}$ is a function of these states, by the data processing inequality it contains no more information about $X_t$ and we have
\begin{align}
    \tCI(M) = \sum_{t=1}^\T I(M_t; X_t\mid M_{t-1}) 
        &= \sum_{t=1}^\T I(M_t; X_t\mid M_{<t})  \\
        &\ge \sum_{t=1}^\T I(P_t; X_t\mid M_{<t}).
\end{align}
The precise notion of ``difference in distributions'' we need is the \textit{Jensen-Shannon divergence}, defined for distributions $p$ and $q$ as
\begin{align}
    \JSD*{p}{q} \defeq \frac{1}{2} \KL*{p}{\frac{p+q}{2}} + \frac{1}{2} \KL*{q}{\frac{p+q}{2}},
\end{align}
where $\KL{\cdot}{\cdot}$ is the Kullback-Leibler divergence (Appendix~\ref{app:preliminaries}).
For uniform random variable $U\in \{0,1\}$ and random variable $A$ that is distributed according to $p$ when $U=0$ and distributed according to $q$ when $U=1$, Fact~\ref{fact:distinguishing} states that we have $\JSD{p}{q} = I(U; A)$.
Thus we have
\begin{align}
     \sum_{t=1}^T I(P_t; X_t\mid M_{<t}) 
        &= \sum_{t=1}^T \E_{m<t}\left[\JSD{\qone}{\qzero}\right], \label{eq:jsd_sum}
\end{align}
Looking at~\eqref{eq:jsd_sum}, we might hope to prove a statement such as $\JSD{\qone}{\qzero} \overset{?}{=}\Omega(|\mone - \mzero|)$.
If this were the case, we would be done: by Lemma~\ref{lemma:posterior_gap_means}, these mean-gaps on average are of size $\delta^2/T$.
However, this statement is false.\footnote{
    As a counterexample, let $p$ and $q$ be the following distributions: $p(0)=q(1) = \frac{1+\alpha}{2}$, and $p(1)=q(0)=\frac{1-\alpha}{2}$, for some $\alpha>0$.
    Then the difference in means is $\alpha$, but a quick calculation shows that $\JSD{p}{q} = O(\alpha^2)$.
}
Instead, we prove the following lower bound on Jensen-Shannon divergence, which is tight up to the leading constant.
Although the proof is elementary, we are not aware of the statement appearing previously.
\newcommand{\jsdlowerbound}{
    Let $\qdist_0$ and $\qdist_1$ be distributions over $\mathbb{R}$ with means $\mu_0$ and $\mu_1$ respectively, and let $\qdist=\frac{\qdist_0+\qdist_1}{2}$.
    Then
    \begin{equation}
        \JSD{\qdist_0}{\qdist_1} \ge \frac 1 8 \cdot  \frac{(\mu_0-\mu_1)^2}{ \mathrm{Var}(\qdist)}.
    \end{equation}
}
\begin{lemma}\label{lemma:JSD_lower_bound}
    \jsdlowerbound
\end{lemma}
Plugging this lower bound into Equation~\eqref{eq:jsd_sum}, we see that our final task is to show that variances are not too large on average.
Formally, we show that an average choice of $t$ and $m_{<t}$ yields a distribution over $P_t$ with variance $O(1/T)$.

\newcommand{\avgvariancefour}{For every algorithm $M$ solving \TaskStream, 
    \[\E \sum_{t=1}^T  \Var{P_t \ \mid\   M_{<t}} 
    \le \tfrac{5}{4}.\] 
    Here $\Var{P_t \ \mid\  m}$ denotes the conditional variance $\E_{P_t}\left[\big(P_t - \E[P_t\mid M_{<t}=m]\big)^2\right]$.
}
\begin{lemma}\label{lemma:avg_variance}
    \avgvariancefour
\end{lemma}
We present the main idea in the proof of Lemma~\ref{lemma:avg_variance}.
Let $\Delta_t = P_t - P_{t-1}$, the increment in the posterior.
Note that $\Var{\Delta_t} = \E_{M_{<t}}[\Var{P_t\mid M_{<t}}]$, since fixing $M_{<t}=m_
{<t}$ also fixed $P_{t-1}=p_{t-1}$.
If $\Delta_1,\ldots,\Delta_\T$ were (pairwise) uncorrelated, then we would have $\Var{P_t} = \sum_{t=1}^\T \Var{\Delta_t}$.
Since $P_\T\in [0,1]$, its variance is at most $\frac{1}{4}$, so the average time step would have $\Var{P_t\mid M_{<t}} = O(1/\T)$.
This argument does not hold for random processes in general: correlations between time steps may introduce ``additional'' variance that does not appear in $P_\T$.
To show that these correlations do not affect our result too much, we first show that the process we consider is decreasing on average (formally, that $P_0, P_1,\ldots, P_T$ form a \textit{supermartingale}).
This implies that the correlations between time steps all ``point in the same direction'' and only contribute a limited amount of additional variance.

\section{A Lower Bound for a Single-Subpopulation Task}\label{sec:subpop_lbd}

We show how to turn the previous lower bound for the one-bit stream task into a lower bound for \TaskSubpopDist, where each example is in $\bit{d}$ and has some number of fixed bits, with the other bits uniformly random.
The number of fixed bits is either $0$ (w.p.~$1/2$) or uniform within $\{0,\ldots, \maxfixed\}$ for some known $\maxfixed$ (otherwise).
Let $M$ be an algorithm for \TaskSubpopDist.
Define
\begin{align}
    \pi_r \defeq \Pr[M = \struct \mid r\text{ fixed bits}].
\end{align}
Observe that this is well-defined even for $r>\maxfixed$.
We construct a learner $M'$ for \TaskStream, the single-stream problem, that takes advantage of the average gaps between $\pi_{r+1}$ and $\pi_{r}$.
$M'$ randomly selects $r\in\{0,\ldots, \rho\}$ and then uniformly picks $r$ indices $\cJ\subseteq [d]$ (without replacement) and values $\cB\in\{0,1\}^r$.
At each time step $t$, $M'$ provides $M$ with an example $z_t$ where the features $j_i\in \cJ$ are fixed to bit $b_i$.
Furthermore, $M'$ selects an index $j_0\in [d]$ at which to insert the one-bit input stream.
Note that if $j_0\in \cJ$ the input stream is ignored; this is necessary for our information argument.
At the end of the stream, $M'$ uses the output of $M$ and knowledge of $\pi_r$ and $\pi_{r+1}$ to produce a guess for whether the one-bit stream is fixed.
See Algorithm~\ref{alg:weak_alg} in Appendix~\ref{app:subpop_lbd} for a formal description of this reduction.

We argue in Appendix~\ref{app:subpop_lbd} that $M'$  has advantage roughly $\delta/\maxfixed$ when the advantage of $M$ is $\delta$.
The proof, in effect, makes use of a hybrid argument: Algorithm~\ref{alg:weak_alg} creates $r$ fake fixed streams for a uniformly random $0\le r\le \maxfixed$.
When the input stream is uniform, the distribution generated in the reduction will have $r$ fixed streams, and when the input stream is fixed the distribution will (usually) have $r+1$ fixed streams.
The minor caveat is that, with probability $\frac{r}{d}$, $M'$ will select the special index $j_0$ to be fixed and ``overwrite'' its own input. 
When $r \ll d$ this does not significantly increase the error.

The above algorithm is carefully constructed so that, \emph{when the input stream is uniform}, the distribution fed to $M$ matches the distribution in \TaskSubpopDist, even conditioned on the location of the input stream $J_0$.
To highlight this crucial fact, we present it separately.
\begin{fact}\label{fact:indep_of_location}
    The random variables $\cJ$ and $\cB$ are independent of $J_0$.
    Since $\cJ$ and $\cB$ determine the distributions of $M_{\le \T}$ and $Z_{\le \T}$ in 
    the algorithm above,
    when the inputs $X_{\le \T}$ are i.i.d.~uniform bits we have
        $J_0 \perp \left(\cJ, \cB, M_{\le \T}, Z_{\le \T}\right)$.
\end{fact}

\begin{lemma}[Information Cost]\label{lemma:one_stream_CI}
    For algorithm $M$ solving \TaskSubpopDist~on a stream of length $\T$, let $M'$ solving \TaskStream~on a stream of length $\T$ be as defined above (see Algorithm~\ref{alg:weak_alg} in Appendix~\ref{app:subpop_lbd} for a formal description).
    Let inputs $X_t\in \{0,1\}$ be uniform and let $Z_t\in \{0,1\}^d$ be structured with parameters $F$ (see Definition~\ref{def:structure}).
    We have
    \begin{align}
        \sum_{t=1}^\T I(M_t'; X_t\mid M_{t-1}') \le \frac{1}{d} \sum_{t=1}^\T I(M_t; Z_t\mid M_{t-1}, F).
    \end{align}
    That is, $\tCI(M\mid F) \geq d \cdot  \tCI(M')$.
\end{lemma}
\begin{proof}
    To generate inputs for $M$, $M'$ stores random variable $J_0$, the location of the input stream, and the locations and values of the fake fixed streams: $\cJ$ and $\cB$.
    Write $F=(\cJ,\cB)$, so the state of $M'$ is $M_t'=(M_t, F, J_0)$ and thus
    \begin{align}
      \sum_{t=1}^\T I(M_t'; X_t\mid M_{t-1}')
            &=  \sum_{t=1}^\T I(M_t; X_t\mid M_{t-1}, F, J_0) \\
            &= \frac{1}{d} \sum_{t=1}^\T \sum_{j=1}^d I(M_t; X_t\mid M_{t-1}, F, J_0=j),
    \end{align}
    writing out the expectation over $J_0=j$.
    Next, note that $M_t$ depends on $X_t$ only through $Z_t^j$.
    Formally, conditioning on $J_0=j$ and any values of $M_{t-1}$ and $F$, we have the Markov chain $M_t \text{---} Z_t^j \text{---} X_t$.
    So we apply the data processing inequality and have
    \begin{align}
        \tCI(M') \le \frac{1}{d} \sum_{t=1}^\T \sum_{j=1}^d I(M_t; Z_t^j \mid M_{t-1}, F, J_0=j).
    \end{align}
    Crucially, this expression involves only terms generated by $M'$ as part of the reduction.
    
    By Fact~\ref{fact:indep_of_location}, we can remove the condition that $J_0=j$, since for any $j$ the tuple $(M_t, Z_t^j, M_{t-1}, F)$ is independent of the location in which $M'$ inserts the input stream.
    We have
    \begin{align}
        \tCI(M')
            \le \frac{1}{d} \sum_{t=1}^\T \sum_{j=1}^d I(M_t; Z_t^j \mid M_{t-1}, F).
    \end{align}
 
    When we condition on $F$, the collection of random variables $Z_t^{<j}$ is independent of $Z_t^j$.
    Thus, via Fact~\ref{fact:add_conditioning}, we can condition on $Z_t^{<j}$ and apply the chain rule over $Z_t=(Z_t^j)_{j\in[k]}$ to finish the proof:
    \begin{align}
        \tCI(MN') &\le \frac{1}{d} \sum_{t=1}^\T \sum_{j=1}^d I(M_t; Z_t^j \mid M_{t-1},F, Z_t^{<j}) \\
        &= \frac{1}{d} \sum_{t=1}^\T I(M_t; Z_t \mid M_{t-1},F) = \frac{1}{d} \tCI(M\mid F).
    \end{align}
\end{proof}

The following lower bound follows directly from the lower bound of Theorem~\ref{theorem:one_bit_lbd} and the reduction presented in this section (Appendix~\ref{app:subpop_lbd}'s accuracy argument in Lemma~\ref{lemma:weak_alg_acc} and the information cost argument in Lemma~\ref{lemma:one_stream_CI}).
\begin{corollary}\label{corollary:tasksubpopdist_hard}
    Any algorithm $M$ solving \TaskSubpopDist~on $\T$ examples with advantage at least $\delta$ satisfies
    \begin{align}
        \tCI(M\mid F) \defeq \sum_{t=1}^T I(M_t; X_t\mid M_{t-1}, F) \ge \frac{d}{40} \left(\frac{\delta}{\maxfixed+1} - \frac{\maxfixed}{d}\right)^4.
    \end{align}
\end{corollary}

\section{A Lower Bound for the Core Problem}\label{sec:core_lbd}

In Section~\ref{sec:subpop_lbd}, we saw that the lower bound for distinguishing uniform-versus-fixed one-bit streams implies a lower bound for distinguishing uniform-versus-structured $d$-bit strings, where the structured distribution is defined in Definition~\ref{def:structure}.
A simple argument shows that this, in turn, implies a lower bound for the associated ``test example'' problem, where the learner is given a stream of structured inputs and then, at test time, is given an example that is either structured or uniform.
We define this problem (\TaskSubpopTestFixed) and present the argument in Appendix~\ref{app:inputs_to_example}.

In this section we sketch the argument that this lower bound implies a lower bound on \TaskCore.
The proof, presented in Appendix~\ref{app:core_lbd}, is delicate, and we believe it may find applications in other random-order streaming problems.

\newcommand{\corelbd}{
    Any algorithm $M$ solving \TaskCore~with parameters $k, d$, and $\maxfixed$ with advantage at least $\delta$ satisfies
    \begin{align*}
        \CI(M\mid S, F) = \sum_{i=1}^\N \sum_{t=i}^{\N} I(M_t; X_i \mid M_{i-1}, S, F) \ge \frac{k^2d}{320} \left(\frac{\delta}{\maxfixed+1} - \frac{\maxfixed}{d}\right)^4 - O(k^2) - O(kd).
    \end{align*}
    Here the inputs $X_i$ are structured (Definition~\ref{def:alltasks}).
}

\begin{theorem}[Lower Bound for \TaskCore]\label{theorem:core_lbd}
    \corelbd
\end{theorem}

Recall that large composable information cost implies $M$ uses a lot of space.
Formally, Lemma~\ref{lemma:space_lbd_SF} states that $\frac{1}{N}\cdot  \CI(M\mid S,F) \le \max_{t\in [\N]} \abs{M_t}$.
Thus, for $\maxfixed=o(d^{1/4})$ and constant $\delta$ we get a space lower bound of $\Omega\paren{\frac{k^2 d}{\N \maxfixed^4}}$.

The proof of this theorem uses a similar reduction to that in \citet{braverman2020coin}'s proof of the $k$-Coins Problem, but our analysis differs significantly.
The reduction uses an algorithm $M$ solving \TaskCore~to construct an algorithm $M'$ solving \TaskSubpopTestFixed~(the single-subpopulation version of the problem). 
We pick a random index $j\in [k]$ and sequence of arrivals $s\in [k]^\N$, insert the input stream whenever $s_t=j$, and generate the other examples from $k-1$ ``synthetic'' subpopulations.

Informally, we show that $\CI(M\mid S, F)$ is $k^2$ times larger than the composable information cost of $M'$.
The first factor of $k$ arises because $M$ does not ``know'' the location of the true input stream: it must solve all subpopulations at once. 

The second factor of $k$ is the crux of the proof, and we present it here.\footnote{We elide several details, including additional conditions in the information terms and technicalities in dealing with the end of the stream.}
Fix a sequence of arrivals up until time $t_0$, when we received an example from subpopulation $j$.
Let $T_1$ denote the time of the next arrival from that subpopulation; it is a random variable.
After some manipulation, we arrive at a sum of terms, each with the form (for some $t\ge t_0$):
\begin{align}
    \Pr[T_1>t]\cdot I(M_t; X_{t_0}\mid M_{t_0-1}, T_1>t). \label{eq:simple_before_change}
\end{align}
Then we observe two facts.
First, we have that $\Pr[T_1 > t] = k\cdot\Pr[T_1=t+1]$, since $T_1-t_0$ is a geometric random variable; at every time step we see an arrival from subpopulation $j$ with probability $\frac{1}{k}$.
Second, since the algorithm cannot depend on events that happen \textit{after} time $t$, it cannot distinguish the event $T_1>t$ from the event $T_1=t+1$, so we have that~\eqref{eq:simple_before_change} is equal to 
\begin{align}
    k \cdot \Pr[T_1=t+1]\cdot I(M_t; X_{t_0}\mid M_{t_0-1}, T_1=t+1). \label{eq:simple_after_change}
\end{align}
This equality introduces the second factor of $k$ and facilitates the move from discussion information about $M$ to information about $M'$, which must account for the information $M$ stores at time $T_1-1$.

\section*{Acknowledgements}

We are grateful to Vitaly Feldman and Kunal Talwar for their insight and feedback throughout this project. In particular,  
they helped articulate conjectures that evolved into the theorems in this paper. The final product would have looked very different without their involvement.

\bibliography{bibliography}

\appendix

\section{Preliminaries}\label{app:preliminaries}

In this section, we present basic facts and tools used in the paper.
For additional definitions and background, see a reference such as \cite{cover1991elements}.
Lemma~\ref{lemma:condition_on_all_states} is new to this paper and Lemma~\ref{lemma:space_lbd_SF} is slight modification of a statement in \citet{braverman2020coin}.

\begin{fact}\label{fact:add_conditioning}
    Let $A, B$, and $C$ be random variables. 
    Suppose $A$ and $C$ are independent. 
    Then $I(A;B) \le I(A;B\mid C)$.
\end{fact}
\begin{fact}\label{fact:remove_conditioning}
    Let $A, B$, and $C$ be random variables. 
    Then $I(A;B\mid C)\le I(A;B) + I(A;C\mid B)$.
\end{fact}
These facts are easily proved by writing out the chain rule two different ways and using the nonnegativity of mutual information:
\begin{align*}
    I(A;B,C) &= I(A;B) + I(A;C\mid B) \\
        &= I(A;C) + I(A;B\mid C).
\end{align*}

A standard property of mutual information is the ``data processing inequality,'' which says that if random variables $X\text{---} Y \text{---} Z$ form a Markov chain, then $I(X;Z) \le I(Y; Z)$.
Streaming algorithms are a special case of this, so we have the following.
\begin{proposition}[DPI for Streaming]\label{proposition:streaming_DPI}
    For any streaming algorithm $M$ (using only private randomness at each time step) run on i.i.d.\ inputs $\{X_t\}$ and any time steps $t_1\le t_2\le t_3$, we have
    \begin{align*}
        I(M_{t_3}; X_{t_1}\mid M_{t_1-1}) \le I(M_{t_2}; X_{t_1}\mid M_{t_1-1}).
    \end{align*}
\end{proposition}

We require several ways to measure differences in probability distributions.
\begin{definition}[Distances and Divergences]\label{def:distances}
    Let $p$ and $q$ be probability distributions over the same space $\cX$.
    Define the following:
    \begin{itemize}
        \item \emph{Kullback-Leibler divergence}: $\mathrm{KL}(p\parallel q)=\sum_{x\in \cX} p(x)\log \frac{p(x)}{q(x)}$.
        \item \emph{Jensen-Shannon divergence}: $\mathrm{JSD}(p\parallel q) = \frac{1}{2} \mathrm{KL}\left( p\parallel \frac{p+q}{2}\right)+\frac{1}{2} \mathrm{KL}\left( q\parallel \frac{p+q}{2}\right)$.
        \item \emph{Total variation distance}: $\mathrm{TV}(p,q) =\frac{1}{2} \sum_{x\in \cX} |p(x)-q(x)|$.
        \item \emph{Squared Hellinger distance}: $H^2(p,q)= 1 - \sum_{x\in \cX} \sqrt{p(x)q(x)}$.
    \end{itemize}
\end{definition}

\begin{fact}\label{fact:hellinger_TV_inequality}
    For any distributions $p$ and $q$, we have $\mathrm{TV}(p,q)\le\sqrt{2}  H(p,q)$.
\end{fact}

We use the fact that two of our measures relate directly to the problem of distinguishing distributions.
\begin{fact}\label{fact:distinguishing}
    Let $B\in\{0,1\}$ be a uniform random variable.
    For distributions $p$ and $q$, let random variable $X$ satisfy $X\sim p$ if $B=0$ and $X\sim q$ otherwise.
    \begin{enumerate}
        \item[(i)] $I(X;B)=\mathrm{JSD}(p\parallel q)$.
        \item[(ii)] Let $\cF$ be the space of functions $f:\cX \to \{0,1\}$.
        We have
        \begin{align*}
            \max_{f\in\cF} \Pr[f(X)=B] = \frac{1+\mathrm{TV}(p,q)}{2}.
        \end{align*}
    \end{enumerate}
\end{fact}

$\tCI(M)$, as defined in Equation~\eqref{eq:tCI_def}, contains information terms that condition on the previous memory state.
Our proof of Theorem~\ref{theorem:one_bit_lbd} uses the following lemma, which says that we may replace this term with one that conditions on all previous memory states.
\begin{lemma}\label{lemma:condition_on_all_states}
    Consider any streaming task where the inputs $X_1, \ldots, X_{\T}$ are independent.
    For any streaming algorithm $M$ and any time step $t\in [\T]$, we have
    \begin{equation}
        I(M_t; X_t\mid M_{t-1}) = I(M_t; X_t\mid M_{<t}). 
    \end{equation}
\end{lemma}
\begin{proof}
    Observe that we have the Markov chain $M_{< t-1}\mbox{---}M_{t-1}\mbox{---}M_t\mbox{---}X_t$.
    Using the chain rule, we have
    \begin{align}
        I(M_{\le t};X_t) &= I(M_{t-1}; X_t) + I(M_t;X_t\mid M_{t-1}) + I(M_{<t-1};X_t\mid M_{t-1},M_t) \\
            &= 0 + I(M_t;X_t\mid M_{t-1}) + 0.
    \end{align}
    The first term is zero because the stream is i.i.d., and the third term is zero because of the Markov chain.
    Applying the chain rule again in a different order, and using the fact that $M_{<t}\perp X_t$, we have
    \begin{align}
        I(M_{\le t};X_t) &= I(M_{<t}; X_t) + I(M_t;X_t\mid M_{< t})\\
            &= 0 + I(M_t;X_t\mid M_{<t}).
    \end{align}
    So the two information terms are equal.
\end{proof}

\begin{lemma}\label{lemma:space_lbd_SF}
    In Task C, the core task, let random variables $S$ and $F$ be the sequence of subpopulation arrivals and subpopulation parameters, respectively.
    For any algorithm $M$ and time step $t\in [\N]$ we have 
    \begin{align}
        \sum_{i=1}^t I(M_t; X_i\mid M_{i-1}, S, F) \le |M_t| \label{eq:one_time_space_bd}
    \end{align} 
    and thus 
    \begin{align}
        \frac{1}{\N}\cdot \CI(M\mid S, F) \le \max_{t\in [\N]} |M_t|. \label{eq:avg_space_bd}
    \end{align}
\end{lemma}

\begin{proof}
    First we show that Equation~\eqref{eq:one_time_space_bd} immediately implies Equation~\eqref{eq:avg_space_bd}.
    Reordering the terms in $\CI(M\mid S, F)$, we see that
    \begin{align*}
        \frac{1}{\N} \cdot \CI(M\mid S, F) &= \frac{1}{\N} \sum_{i=1}^{\N}\sum_{t=i}^{\N} I(M_t; X_i\mid M_{i-1}, S, F) \\
            &=\frac{1}{\N}  \sum_{t=1}^{\N}\sum_{i=1}^{t} I(M_t; X_i\mid M_{i-1}, S, F) \\
            &\le \E_t\left[|M_t|\right] \le \max_{t\in [\N]} |M_t|.
    \end{align*}
    
    It remains to prove Equation~\eqref{eq:one_time_space_bd}.
    We have $I(X_i; M_{<i-1}, X_{<i}\mid M_{i-1}, S, F)=0$ (since $X_i$ depends only on $F$ and $S$). 
    We apply Fact~\ref{fact:add_conditioning} and have the inequality
    \begin{align}
        \sum_{i=1}^t I(M_t; X_i\mid M_{i-1}, S, F) &\le \sum_{i=1}^t I(M_t; X_i\mid M_{i-1}, M_{<i-1}, X_{<i} S, F).
    \end{align}
    Next we add another (nonnegative) mutual information term and apply the chain rule twice:
    \begin{align}
        \sum_{i=1}^t I(M_t; X_i\mid M_{i-1}, S, F) &\le \sum_{i=1}^t \biggl[ I(M_t; X_i\mid M_{i-1}, M_{<i-1}, X_{<i}, S, F) \\
            &\hspace{1cm}+  I(M_t; M_{i-1}\mid M_{<i-1}, X_{<i}, S, F) \biggr] \\
            &= \sum_{i=1}^t I(M_t; X_i, M_{i-1} \mid M_{<i-1}, X_{<i}, S, F) \\
            &= I(M_t; X_{\le i}, M_{\le i-1} \mid S, F).
    \end{align}
    This term is at most $H(M_t\mid S, F)$, which is at most $H(M_t)\le |M_t|$.
\end{proof}

\section{Proofs for Section~\ref{sec:taskstream_lbd}}\label{app:taskstream_lbd_proofs}

Recall that, for algorithm $M$ with advantage $\delta$, our analysis considers an observer who, at time $t$, sees all of $M$'s previous memory states $m_{\le t}$.
We consider running $M$ on either the 1's stream or the uniform stream, and analyze the observer's posterior belief $\{p_t\}$ about $E$, the event that the stream is fixed to all 1's. 
Restricting our focus to the 1's stream (as opposed to 0's) versus uniform is without loss of generality, as we argued above.

We first prove that, on average, the final posterior $P_T$ has decreased significantly. 
Recall that we fix prior $P_0=\frac{1}{2}$.
\begin{lemma}[Lemma~\ref{lemma:final_posterior_gap} Restated]
    \finalposteriorgap
\end{lemma}
\begin{proof}
    Let $f_1$ denote the joint distribution over memory states $m_{\le T}$ when the input is fixed to 1's, and let $f_{u}$ denote the distribution when the input is uniform.
    By Bayes rule, for any fixed $m_{\le T}$ we have
    \begin{align*}
        p_t = \Pr[E\mid M_{\le T}=m_{\le T}] 
            &= \frac{f_1(m_{\le T}) \cdot \frac{1}{2}}{f_u(m_{\le T}) \cdot \frac{1}{2} + f_1(m_{\le T}) \cdot \frac{1}{2}}
    \end{align*}
    and the expectation under the uniform inputs is 
    \begin{align*}
        \E[P_T] = \frac{1}{2}\sum_{m} \frac{f_u(m) f_1(m)}{\frac{1}{2}(f_1(m) + f_u(m))}.
    \end{align*}
    Writing $f_u(m) f_1(m) = \sqrt{f_u(m)f_1(m)} \sqrt{f_u(m)f_1(m)}$, we apply the arithmetic mean/geometric mean inequality and have
    \begin{align*}
        \E[P_t] \le \frac{1}{2} \sum_{m}  \sqrt{f_u(m)f_1(m)} = \frac{1}{2} \left(1- H^2(f_u, f_1)\right),
    \end{align*}
    where we have introduced the squared Hellinger distance (Definition~\ref{def:distances}).
    By Fact~\ref{fact:hellinger_TV_inequality}, the Hellinger-total variation inequality, we arrive at
    \begin{align*}
        \E[P_T] \le \frac{1}{2} - \frac{1}{4}\mathrm{TV}^2(f_u,f_1).
    \end{align*}
    This suffices to finish the proof: the TV distance between $f_u$ and $f_1$ is at least $\delta$.
    This follows from the $\delta$-advantage assumption on $M$ (stated in Equation~\eqref{eq:one_bit_advantage}) and Fact~\ref{fact:distinguishing} part (ii), which implies that access to the states $m_{<T}$ only increases the space of distinguishing functions and thus can only increase the TV distance.
\end{proof}

For any time $t$ and fixed set of previous memory states $m_{<t}$, the distribution over $P_t$ is a mixture of two distributions: $\qone$, arising when $X_t=1$, and $\qzero$, otherwise.
We show that, when $P_t$ has a large change in expectation from $p_{t-1}$, the means of these two distributions are far apart.
The core of the proof is a convexity argument showing that, when we have input $X_t=1$, the posterior belief in the stream being fixed increases on average.
\begin{lemma}[Lemma~\ref{lemma:posterior_gap_means} Restated]
    \posteriorgapmeans
\end{lemma}
\begin{proof}
    We have $\E[P_t\mid M_{<t}=m_{<t}] = \frac{1}{2}\paren{\mzero + \mone}$.
    Fixing $M_{<t}=m_{<t}$ also fixes $P_{t-1}=p_{t-1}$ for some value $p_{t-1}$, so it suffices to prove
    \begin{align}
        p_{t-1} \le \frac{\mzero - \mone}{2} - \frac{\mzero + \mone}{2} = \mone.
    \end{align}

    Let $f_0$ be the distribution over memory states at time $t$ when the input is $X_t=0$, and define $f_1$ similarly.
    With this notation, and recalling that event $E$ means the input stream is fixed to 1's, we can write the posterior as
    \begin{align*}
        p_t = \Pr[E\mid M_{\le t}=m_{\le t}] &= \frac{f_1(m_t) \cdot p_{t-1}}{f_1(m_t) \cdot p_{t-1}+ \frac{f_1(m_t) + f_0(m_t)}{2}\cdot (1 - p_{t-1})} \\
         &= \frac{2\cdot f_1(m_t) \cdot p_{t-1}}{(1+p_{t-1})f_1(m_t) + (1-p_{t-1})f_0(m_t)} \\
         &= \frac{2\cdot f_1(m_t) \cdot p_{t-1}}{(f_1(m_t)+f_0(m_t)) + p_{t-1} (f_1(m_t) - f_0(m_t))}. 
    \end{align*}
    
    Considering the expectation of the posterior conditioned on $X_t=1$, we have
    \begin{align*}
         \mu_t^{(1)} = \E_{f_1}\left[\frac{2 p_{t-1}}{1+p_{t-1} + (1-p_{t-1})\cdot\frac{f_0(m_t)}{f_1(m_t)}}\right] 
            &\ge \frac{2 p_{t-1}}{1+p_{t-1} + (1-p_{t-1})\cdot\E_{f_1}\left[\frac{f_0(m_t)}{f_1(m_t)}\right]} \\
            &= \frac{2 p_{t-1}}{1+p_{t-1} + (1-p_{t-1})\cdot 1} \\
            &= p_{t-1},
    \end{align*}
    The inequality use Jensen's inequality: for any value $p\in [0,1]$, $\frac{1}{1+p + (1-p)\cdot \alpha}$ is convex in $\alpha$.
    To see this, let $g(\alpha)=(1+p+\alpha-p\alpha)^{-1}$. 
    Then $g'(\alpha)=\frac{p-1}{(1+p+\alpha-p\alpha)^{2}}$ and $g''(\alpha)=\frac{2(p-1)^2}{(1+p+\alpha-p\alpha)^3}$.
    The numerator is nonnegative.
    Since $\alpha\ge p\alpha$, the denominator of $g''$ is positive.
\end{proof}

Next we lower bound the Jensen-Shannon divergence of two distributions in terms of the difference of their means and the variance of their (uniform) mixture.
For simplicity we state the proof for discrete distributions.
Since both distributions are absolutely continuous with respect to their mixture, the continuous case is analogous.
\begin{lemma}[Lemma~\ref{lemma:JSD_lower_bound} Restated]
    \jsdlowerbound
\end{lemma}
\begin{proof}
    \newcommand{\mixdist}{\ensuremath{\qdist}}
    Define $\gamma(x) = \frac{\qdist_1(x)-\mixdist(x)}{\mixdist(x)}= \frac{\qdist_1(x)-\qdist_0(x)}{2\mixdist(x)}$. 
    We have $\qdist_0=(1-\gamma)\qdist$ and $\qdist_1=(1+\gamma)\mixdist$.
    \begin{align}
         \JSD{\qdist_0}{\qdist_1} &= \frac{1}{2}\sum_x \qdist_0(x) \ln \frac{\qdist_0(x)}{\mixdist(x)} + \qdist_1(x) \ln \frac{\qdist_1(x)}{\mixdist(x)} \\
        &= \frac{1}{2} \sum_x \mixdist(x)\bigl(\left(1-\gamma(x)\right)\ln \left(1-\gamma(x)\right) +\left(1+\gamma(x)\right)\ln \left(1+\gamma(x)\right) \bigr) \\
        &\ge \frac{1}{2} \sum_x \mixdist(x) \gamma(x)^2 
        = \frac{1}{2} \sum_x \mixdist(x) \left(\frac{\qdist_1(x)-\qdist_0(x)}{2\cdot \mixdist(x)}\right)^2. \label{eq:JSD_ubd}
    \end{align}
    The inequality applies the 
    fact\footnote{To see this, take derivatives. Let $f(\gamma)= \left(1-\gamma\right)\ln \left(1-\gamma\right) +\left(1+\gamma\right)\ln \left(1+\gamma\right)$. We have $f'(\gamma) = \ln \frac{1+\gamma}{1-\gamma}$ and $f''(\gamma)=\frac{2}{1-x^2}\ge 2$, while $\frac{d^2}{dx^2}x^2 = 2$, and $f(0)=0^2 = 0$.} 
    that, for $\gamma\in (-1,1)$, $\left(1-\gamma\right)\ln \left(1-\gamma\right) +\left(1+\gamma\right)\ln \left(1+\gamma\right) \ge \gamma^2$. 
    Define $\mu = \frac{1}{2}(\mu_0 + \mu_1)$, the mean of $\mixdist$.
    We have
    \begin{align}
        \mu_1 - \mu_0 
            &= \E_{X\sim \qdist_1}[X-\mu] - \E_{X\sim \qdist_0}[X-\mu] \\ 
            &= \sum_x (\qdist_1(x)-\qdist_0(x)) (x-\mu) \\
            &= \E_{X\sim \mixdist}\left[ (X-\mu)\cdot \frac{\qdist_1(X)-\qdist_0(X)}{\mixdist(X)} \right] \\
            &\le \sqrt{\E_{X\sim\mixdist}\left[ (X-\mu)^2 \right]}
                \sqrt{\E_{X\sim\mixdist}\left[ \left( \frac{\qdist_1(X)-\qdist_0(X)}{\mixdist(X)} \right)^2\right]},
    \end{align}
    applying Cauchy-Schwarz. 
    By definition, $\E_{X\sim\mixdist}\left[ (X-\mu)^2 \right] = \mathrm{Var}(\mixdist)$.
    Plugging in the bound from Equation~\ref{eq:JSD_ubd}, we have
    \begin{equation}
        \mu_1 - \mu_0 \le \sqrt{\mathrm{Var}(\mixdist)} \sqrt{8\cdot\JSD{\qdist_0}{\qdist_1}}.
    \end{equation}
    Taking squares and rearranging finishes the proof.
\end{proof}

Before proving Theorem~\ref{theorem:one_bit_lbd}, we show that the average variances are not too large.
The first step in this direction is to show that, when the inputs are uniform, the sequence of posteriors forms a supermartingale.
The proof uses a convexity argument similar to that of Lemma~\ref{lemma:posterior_gap_means}.
\begin{lemma}\label{lemma:supermartingale}
    Let random process $P_0, P_1, \ldots, P_T$ be the sequence of posteriors when the inputs $X_t$ are i.i.d.\ uniform.
    Then, for any $t$ and past memory states $m_{<t}$, 
    \begin{equation}
        \E[P_t\mid M_{<t}=m_{<t}] \le p_{t-1}.
    \end{equation}
\end{lemma}
\begin{proof}
    Throughout this proof, leave the conditioning on $m_{<t}$ implicit.
    Let $f_0$ be the distribution over memory states at time $t$ when the input is $X_t=0$, and define $f_1$ similarly.
    With this notation, and recalling that event $E$ means the input stream is fixed to 1's, we can write the posterior as
    \begin{align*}
        p_t = \Pr[E\mid M_{\le t}=m_{\le t}] &= \frac{f_1(m_t) \cdot p_{t-1}}{f_1(m_t) \cdot p_{t-1}+ \frac{f_1(m_t) + f_0(m_t)}{2}\cdot (1 - p_{t-1})} \\
         &= \frac{2\cdot f_1(m_t) \cdot p_{t-1}}{(1+p_{t-1})f_1(m_t) + (1-p_{t-1})f_0(m_t)} \\
         &= \frac{2\cdot f_1(m_t) \cdot p_{t-1}}{(f_1(m_t)+f_0(m_t)) + p_{t-1} (f_1(m_t) - f_0(m_t))}. 
    \end{align*}
        
    We consider the expectation over uniform inputs, rewriting it as an expectation over $f_1$:
    \begin{align*}
        \frac{\mu_t^{(0)} + \mu_t^{(1)}}{2} 
            &= \E_{\frac{f_0+f_1}{2}}\left[\frac{2 f_1(m_t)p_{t-1}}{(f_1(m_t)+f_0(m_t)) + p_{t-1}(f_1(m_t)-f_0(m_t))}\right] \\
            &= \sum_{m_t} \left(\frac{f_1(m_t)+f_0(m_t)}{2}\right) \left(\frac{2 f_1(m_t)p_{t-1}}{(f_1(m_t)+f_0(m_t)) + p_{t-1}(f_1(m_t)-f_0(m_t))} \right)\\
            &= \E_{f_1}\left[\frac{ p_{t-1}}{1 + p_{t-1}\cdot\frac{f_1(m_t)-f_0(m_t)}{f_1(m_t)+f_0(m_t)}}\right] \\
            &= \E_{f_1}\left[\frac{ p_{t-1}}{1 + p_{t-1}\cdot\frac{1-f_0(m_t)/f_1(m_t)}{1+f_0(m_t)/f_1(m_t)}}\right].
    \end{align*}
    For any fixed $p\in[0,1]$, $\frac{1}{1+p\cdot \frac{1-\alpha}{1+\alpha}}$ is concave for $\alpha\in [0,1]$.
    To see this, let $h(\alpha)=\frac{1+\alpha}{1+\alpha+p-p\alpha}$.
    Then $h'(\alpha)= \frac{2p}{(1+\alpha+p-p\alpha)^2}$ and $h''(\alpha)=\frac{-4(1-p)p}{(1+\alpha+p-p\alpha)^3}$.
    The denominator is positive and the numerator is nonpositive.
    
    So we have via Jensen's inequality that
    \begin{align*}
        \frac{\mu_t^{(0)} + \mu_t^{(1)}}{2} &\le \frac{ p_{t-1}}{1 + p_{t-1}\cdot\E_{f_1}\left[\frac{1-f_0(m_t)/f_1(m_t)}{1+f_0(m_t)/f_1(m_t)}\right]} 
        \le \frac{ p_{t-1}}{1 + 0}.
    \end{align*}
    To see the second inequality, note that the denominator is always at most some constant (that depends on $f_1$) and that $\E_{f_1}[1-f_0(m_t)/f_1(m_t)] = 0$.
\end{proof}

We now upper bound the average ``expected variance'' of $P_t$, i.e., the variance of $P_t$ on average when $M_{<t}$ is fixed.
\begin{lemma}[Lemma~\ref{lemma:avg_variance} Restated]
    \avgvariancefour
\end{lemma}
\begin{proof}
    Define random variable $\Delta_t \defeq P_t - P_{t-1}$, the difference in posteriors, and write $P_{\T} = P_0 + \sum_{t=1}^{\T} \Delta_t$.
    We will recursively apply the following elementary variance decomposition.
    \begin{claim}\label{claim:distribute_variance}
        Let $A,B$, and $C$ be random variables.
        We have
        \begin{align}
            \Var{A+B} = \Var{A} + \E_{B,C}\left[\paren{B - \E[B \mid C]}^2\right] + \Var{\E[B\mid C]} + 2\cdot \mathrm{Cov}\paren{A, \E[B\mid C]}.
        \end{align}
    \end{claim}
    \begin{proof}[Proof of Claim~\ref{claim:distribute_variance}]
        First, $\Var{A+B} = \Var{A}+\Var{B}+2\cdot\mathrm{Cov}(A,B)$.
        Next, we add and subtract $\E[B\mid C]$ and expand:
        \begin{align}
            \Var{B}=\E_C[\Var{B}] &= \E_{B,C}[(B-\E[B])^2] \\
                &= \E_{B,C}[(B- \E[B\mid C] + \E[B\mid C] - \E[B])^2] \\
                &= \E_{B,C}\bigl[(B-\E[B\mid C])^2 + (\E[B\mid C] - \E[B])^2]  \\
                &\quad + \E_{B,C}[2(B-\E[B\mid C])(\E[B\mid C]-\E[B])] \\
                &= \E_{C,B}[(B-\E[B\mid C])^2] + \Var{\E[B\mid C]} + 0,
        \end{align}
        noting that we have $\E[\E[B\mid C] -\E[B]]=0$ for any $B=b$.
        To finish, we again add and subtract $\E[B\mid C]$ and expand:
        \begin{align}
            \mathrm{Cov}(A,B) &= \E_{A,B,C}[(A-\E[A])(B-\E[B])] \\
                &= \E_A\left[(A-\E[A])\E_{B,C}[B-\E[B\mid C]+\E[B\mid C]-\E[B]]\right] \\
                &=\E_A\left[(A-\E[A])\E_{B,C}[0 + \E[B\mid C]-\E[B]]\right] = \mathrm{Cov}(A, \E[B\mid C]), 
        \end{align}
        again using the fact that $\E[B\mid C] = \E[B]$.
    \end{proof}

    Repeatedly applying the claim to the sum $P_0 + \sum_{t=1}^T\Delta_t$, we get
    \begin{align}
        \Var{P_{\T}} &= \Var{P_0} + \sum_{t=1}^\T\Bigl( \E_{M_{<t},\Delta_t}[(\Delta_t - \E[\Delta_t\mid M_{<t})^2] + \Var{\E[\Delta_t\mid M_{<t}} \\
            &\quad + 2\cdot \mathrm{Cov}(P_{t-1} , \E[\Delta_t\mid M_{<t}]) \Bigr) \label{eq:one_bit_variance_long} \\
        &\ge \sum_{t=1}^\T \E_{M_{<t},\Delta_t}[(\Delta_t - \E[\Delta_t\mid M_{<t}])^2] + 2\cdot\sum_{t=1}^\T \mathrm{Cov}(P_{t-1} , \E[\Delta_t\mid M_{<t}]), \label{eq:unrolled_variances}   
    \end{align} 
    since the variance terms are nonnegative.
    Recall that conditioning on $m_{<t}$ fixes $p_{t-1}$, so 
    \begin{align}
        \E_{M_{<t},\Delta_t}[(\Delta_t - \E[\Delta_t\mid M_{<t}])^2] = \E_{M_{<t},P_t}[(P_t - \E[P_t\mid M_{<t}])^2].
    \end{align}
    Thus it remains to lower bound the sum of the covariance terms (which may be negative).

    We begin with the fact that, for any random variables $A$ and $B$, $\mathrm{Cov}(A,B)=\E[(A-\E[A])B]$.
    We then apply two facts about the sequence of posteriors: first, that $\E[\Delta_t\mid M_{<t}]$ is nonpositive, and second, that $(P_{t-1} - \E[P_{t-1}])$ has $1$ as an upper bound.
    (The former fact we proved in Lemma~\ref{lemma:supermartingale}.)
    \begin{align}
        \sum_{t=1}^T \mathrm{Cov}(P_{t-1}, \E[\Delta_t\mid M_{<t}]) 
            &= \sum_{t=1}^T \E\left[(P_{t-1}-\E[P_{t-1}]) \E[\Delta_t\mid M_{<t}]\right] \\
            &\ge \sum_{t=1}^T \E\left[ \E[\Delta_t\mid M_{<t}]\right]
            = \sum_{t=1}^T \E\left[ \Delta_t\right].
    \end{align}
    By linearity of expectation, this sum is exactly $\E\left[\sum_t \Delta_t\right] = \E[P_T - P_0]$.
    Since $P_0=1/2$ this expectation is at least $-1/2$, so we have, recalling the factor of 2 in front of the covariance sum,
    \begin{align}
        \frac{1}{4} \ge \Var{P_\T} \ge \sum_{t=1}^T \E_{M_{<t},P_t}[(P_t - \E[P_t\mid M_{<t}])^2] - 1.
    \end{align}
    Adding $1$ to both sides finishes the proof.
\end{proof}

\begin{theorem}[Theorem~\ref{theorem:one_bit_lbd} Restated]
    \onebitlbd
\end{theorem}
\begin{proof}
    Recall from Equation~\ref{eq:jsd_sum} that
    \begin{align}
        \sum_{t=1}^\T I(M_t; X_t\mid M_{<t})
            \ge \sum_{t=1}^\T \E_{M_{<t}}\left[ \JSD*{\qzero}{\qone}\right].
    \end{align}
    For clarity, denote $\JSD*{\qzero}{\qone} \defeq \mathrm{JSD}_t$ and $\E_{P_t}\left[(P_t - \E[P_t\mid M_{<t}])^2\right]\defeq \mathrm{Var}_t$.
    Both of these are random variables that depend on $M_{<t}$; the latter is exactly the variance of the mixture $\frac{1}{2}\paren{\qzero + \qone}$.

    Taking expectations over time steps and then over $M_{<t}$, we apply Cauchy-Schwarz and then Lemma~\ref{lemma:JSD_lower_bound}, our lower bound on the Jensen-Shannon divergence:
    \begin{align}
        \left(\E_t \E_{M_{<t}}\left[\ \mathrm{JSD}_t \right]\right)\left(\E_t \E_{M_{<t}}\left[\ \mathrm{Var}_t \right]\right)
            &\ge \left( \E_t\E_{M_<t}\left[ \sqrt{ \mathrm{JSD}_t\cdot \mathrm{Var}_t} \right] \right)^2 \\
            &\ge \frac{1}{8} \left( \E_t\E_{M_<t}\left[ |\mu_t^{(0)} - \mu_t^{(1)}| \right] \right)^2.
    \end{align}
    On average, the difference between $\mzero$ and $\mone$ is appreciable: we apply Jensen's inequality to move the absolute value outside the expectation and apply Lemma~\ref{lemma:posterior_gap_means}:
    \begin{align}
        \left(\E_t \E_{M_{<t}}\left[\ \mathrm{JSD}_t \right]\right)\left(\E_t \E_{M_{<t}}\left[\ \mathrm{Var}_t \right]\right)
            &\ge \frac{1}{8} \left( \frac{1}{\T}\sum_{t=1}^\T \E_{M_{<t}}\left[ \mzero - \mone \right]  \right)^2 \\
            &\ge \frac{1}{2} \left( \frac{1}{\T}\sum_{t=1}^\T \E_{M_{<t}}\left[ \E[P_t - P_{t-1}\mid M_{<t}] \right]  \right)^2.
    \end{align}
    (Note that this produces a factor of 2 inside the square.)
    We use the fact that $\E[\E[X\mid Y]] = \E[X]$ for random variables $X$ and $Y$ and cancel out the intermediate terms in the sum, arriving at
    \begin{align}
        \left(\E_t \E_{M_{<t}}\left[\ \mathrm{JSD}_t \right]\right)\left(\E_t \E_{M_{<t}}\left[\ \mathrm{Var}_t \right]\right)
            &\ge \frac{1}{2} \left( \frac{1}{\T} \E[P_T - P_0]  \right)^2 
            \ge \frac{1}{32} \cdot \frac{\delta^4}{\T}.
    \end{align}
    The final inequality uses Lemma~\ref{lemma:final_posterior_gap}: the algorithm has advantage $\delta$, so $\E[P_T - P_0] \le -\frac{\delta^2}{4}$. 

    Rewriting the expectations as sums over $t$ and multiplying by $\T^2$ on both sides, we have
    \begin{align}
        \left(\sum_{t=1}^\T\E\left[ \mathrm{JSD}_t \right]\right)\left(\sum_{t=1}^\T \E\left[ \mathrm{Var}_t\right]\right) \ge \frac{\delta^4}{32}.
    \end{align}
    Lemma~\ref{lemma:avg_variance} says the sum of variances is at most $\frac{5}{4}$, so we have 
        $\tCI(M) =  \sum_{t=1}^\T \mathrm{JSD}_t \ge  \frac{\delta^4}{40}$.
\end{proof}

\section{Proofs for Section~\ref{sec:subpop_lbd}}\label{app:subpop_lbd}

We state the formal reduction in Algorithm~\ref{alg:weak_alg}.
The algorithm assumes knowledge of the exact values of $\pi_{r}$ for all $r\in \{0,\ldots,\maxfixed+1\}$.
These values can in principle be computed to arbitrary accuracy given access to $M$; since our argument is information-theoretic we can ignore the computational concerns.

\begin{algorithm2e}[H]\label{alg:weak_alg}
    \SetAlgoLined
    \SetKwInOut{Input}{input}

    \Input{stream $x_1,\ldots, x_{\T}\in\{0,1\}$; algorithm $M$ for \TaskSubpopDist;  parameters $d$ and $\rho \le d$.}
    \BlankLine
    Draw $r\sim \{0,\ldots, \rho\}$ uniformly\tcc*{number of ``fake'' fixed streams}
    Draw $\cJ=\{j_1,\ldots, j_{r}\}\subseteq [d]$ w/o replacement\tcc*{indices of fixed streams}
    Draw $\cB=(b_1,\ldots, b_{r})\in \{0,1\}^{r}$ uniformly\tcc*{values of fixed streams}
    Draw $j_0 \in [d]$ uniformly\tcc*{location of input stream}
    \For{$t=1,\ldots, \T$}{
        Receive sample $x_t\in\{0,1\}$\;
        Draw $z_t\in\{0,1\}^d$ uniformly\;
        Set $z_{j_0}^t\gets x^t$\;
        $\forall j_\ell\in \cJ$, set $z_{j_\ell}^t \gets b_{j_\ell}$\tcc*{if $j_0\in \cJ$, overwrites $x_t$}
        Execute $M(z^t)$\;
    }
    Receive output of $M$, $\texttt{OUT} \in \{\struct, \unif\}$\;
    \eIf{$\pi_{r+1} \ge \pi_{r}$}{
        \KwRet $\texttt{OUT}$\;
    }{
        \textbf{return} $\lnot \texttt{OUT}$\tcc*{other entry in $\{\struct, \unif\}$}
    }
\caption{$M'$ for \TaskStream}
\end{algorithm2e}

\begin{lemma}\label{lemma:average_gaps_r0}
    Suppose $M$ for \TaskSubpopDist~has advantage at least $\delta$.
    Let $R\sim\mathrm{Uniform}(\{0,1\ldots, \rho\})$.
    Then $\E[|\pi_{R+1} - \pi_{R}|] \ge \frac{\delta}{\rho+1}$.
\end{lemma}
\begin{proof}
    Let $r^*$ be the number with the highest probability of returning $\struct$, i.e., $r^*=\mathrm{argmax}_{r\in [\maxfixed]} \pi_r$.\footnote{With advantage $\delta>0$, we know $r^*\ge 1$.}
    By the advantage assumption, $\pi_{r^*}-\pi_0 \ge \delta$.
    We discard the terms beyond $r^*$ and apply the triangle inequality:
    \begin{align}
        \E[|\pi_{R+1} - \pi_{R}|] 
            &= \frac{1}{\rho+1} \sum_{r=0}^{\rho} |\pi_{r+1} - \pi_{r}| \\
            &\ge \frac{1}{\rho+1} \sum_{r=0}^{r^*-1} |\pi_{r+1} - \pi_{r}| \\
            &\ge \frac{1}{\rho+1} \left|\sum_{r=0}^{r^*-1} \pi_{r+1} - \pi_{r}\right|.
    \end{align}
    The intermediate terms in the sum cancel out and we are left with $|\pi_{r^*} - \pi_0| \ge \delta$.
\end{proof}

\begin{lemma}[Accuracy of Algorithm~\ref{alg:weak_alg}]\label{lemma:weak_alg_acc}
    Assume $M$ has advantage at least $\delta$.
    Then $M'$ has advantage at least 
    \begin{align*}
        \frac{\delta}{\maxfixed+1} + \frac{\maxfixed}{d}
    \end{align*}
\end{lemma}
\begin{proof}
    Recall the definition of \TaskStream: the distribution $\cU$ generates uniform bits and the structured distribution selects a random bit $b\in \{0,1\}$ and fixes all values to $b$
    We wish to show that
    \begin{align*}
        \Pr_{x_{\le \T}\sim \cP_{\mathrm{bit}}}[M'(x_{\le \T})=\struct]-
        \Pr_{x_{\le \T}\sim \cU}[M'(x_{\le \T})=\struct] \ge \frac{\delta}{\maxfixed+1} - \frac{\maxfixed}{d}.
    \end{align*}
    To do this, we show that, on both the structured distribution and $\cU$, $M'$ is correct with probability at least $\frac{1}{2}+\frac{1}{2}\left(\frac{\delta}{\maxfixed+1} - \frac{\maxfixed}{d}\right)$.
    
    $M'$ randomly selects $r\in \{0,\ldots, \maxfixed\}$ as the number of synthetic fixed streams to feed to $M$.
    Fix some $r$: we lower bound the total variation distance between the two output distributions of $M$ (corresponding to structured and uniform inputs to $M'$).
    When the input stream is uniform, $M$ receives synthetic inputs $z^t$ with exactly $r$ fixed streams, so outputs \struct~with probability $\pi_r$.
    When the input stream is structured, $M$ receives synthetic inputs $z^t$ with $r+1$ fixed streams unless $j_0\in \cJ$, i.e. the input stream is overwritten.
    This only happens with probability $\frac{r}{d}$, so when the input stream is structured $M$ outputs \struct~with probability $\left(1 - \frac{r}{d}\right)\pi_{r+1} + \frac{r}{d}\cdot \pi_r$.
    Thus, for any $r$, the total variation distance is at least
    \begin{align*}
        \left|\left(1 - \frac{r}{d}\right)\pi_{r+1} + \frac{r}{d}\cdot \pi_r  - \pi_r \right|
            \ge |\pi_{r+1}-\pi_r| - \frac{\maxfixed}{d},
    \end{align*}
    using the fact that $r\le \maxfixed$ and $\pi_{r+1},\pi_{r}\in [0,1]$.
    
    By Fact~\ref{fact:distinguishing}, this lower bound on total variation distance implies a lower bound on the accuracy of $M'$.
    Since the accuracy of $M'$ is an average over $r$, we have by linearity of expectation that
    \begin{align*}
        \Pr[M'\text{ correct}] = \frac{1}{2} + \frac{1}{2}\left(\E_R[|\pi_{R+1} - \pi_R|] - \frac{\maxfixed}{d}\right).
    \end{align*}
    By Lemma~\ref{lemma:average_gaps_r0}, the expectation is at least $\frac{\delta}{\maxfixed+1}$.
\end{proof}

\section{Proofs: From Distinguishing Inputs to Distinguishing a Test Example}\label{app:inputs_to_example}

In this section, we move between two versions of the ``single subpopulation'' task: in \TaskSubpopTestFixed~the learner receives $T$ structured inputs and must distinguish a structured test example from a uniform one, while in \TaskSubpopDist~the learner must determine whether its inputs are all structured or all uniform.
Recall that $\cU$ denotes the uniform distribution.

\begin{definition}[\TaskSubpopTestFixed]
    Parameters: positive integers $\T, d, \maxfixed\le d$.
    The learning algorithm $M$ receives a stream of $\T$ i.i.d.~samples from $\cP$.
    After $\T$ time steps, the learner outputs a (possibly randomized) function $m:\bit{d}\to \{\unif,\struct\}$.
    The \emph{advantage} of the learner is
    \begin{align*}
        \E_{\substack{\cP\sim \cQ_{d,\maxfixed}\\
                      x=(x_1,\ldots, x_\T)\simiid \cP\\
                      m\gets M(x)}}
            \left[ \Pr_{y\sim \cP}[m(y)=\struct] - \Pr_{y\sim \cU}[m(y)=\struct]\right].
    \end{align*}
\end{definition}

\begin{lemma}
    For any algorithm $M$ solving \TaskSubpopTestFixed~on $\T$ examples with advantage at least $\delta$, there is an algorithm $M'$ solving \TaskSubpopDist~on $\T+1$ examples with advantage at least $\delta/2$ and satisfying
    \begin{align}
        \sum_{t=1}^{\T+1} I(M_t'; X_t\mid M_{t-1}', F) 
                \le \sum_{t=1}^{\T} I(M_t; X_t\mid M_{t-1}, F) + 1.
    \end{align}
    Here the inputs $X_{t}$ are drawn from a structured distribution with parameters $F$ (see Definition~\ref{def:structure}).
\end{lemma}
This information inequality also holds when the $X_{t}$ are i.i.d.~uniform, but we only require the result for structured distributions.
\begin{proof}
In this proof, let $X_i$ denote the structured inputs drawn from $\cP$ (see Definition~\ref{def:structure}) and $U_i$ denote i.i.d.~uniform inputs.
\TaskSubpopTestFixed~on $\T$ examples asks the learner to distinguish a test example, i.e.,
\begin{align}
    (X_{\le \T}, X_{\T+1}) \hspace{0.5cm}\text{from}\hspace{0.5cm} (X_{\le \T}, U_{\T+1}),
\end{align}
and \TaskSubpopDist~on $\T+1$ samples asks the learner to distinguish its inputs:
\begin{align}
    X_{\le \T+1}=(X_{\le \T}, X_{\T+1}) \hspace{0.5cm}\text{from}\hspace{0.5cm} U_{\le \T+1}=(U_{\le \T}, U_{\T+1}).
\end{align}

Consider an algorithm $M$ solving \TaskSubpopTestFixed~and write it $M(\cdot, \cdot)$, as it takes as input a stream of length $\T$ and an additional test example.
Define
\begin{align*}
    \Pr[M(X_{\le\T}, X_{\T+1})=\struct] &= p_{x,x} \\
    \Pr[M(X_{\le\T}, U_{\T+1})=\struct] &= p_{x,u}.
\end{align*}
Since $M$ has advantage at least $\delta$, by definition we have $p_{x,x} - p_{x,u}\ge \delta$.
Now, there is some probability $p_{u,u}$ such that
\begin{equation*}
    \Pr[M(U_{\le\T}, U_{\T+1})=\struct] = p_{u,u}. 
\end{equation*}
Using the fact that $p_{u,u}$ cannot be close to both $p_{x,x}$ and $p_{x,u}$, we design the algorithm $M'$ for \TaskSubpopDist~depending on which it is closer to.

\textbf{Case 1:} If $p_{x,x}-p_{u,u}\ge \delta/2$, $M'$ simply runs $M$ and outputs $M$'s answer.
This has advantage $p_{x,x}-p_{u,u}\ge \delta/2$.

\textbf{Case 2:} If $p_{x,x}-p_{u,u} < \delta/2$, then we know that $p_{u,u}-p_{x,u}\ge \delta/2$, which again gives us a way to distinguish the two inputs.
In this case, $M'$ runs $M$ through time $\T$ and then generates a fresh uniform example $U_{\T+1}$ to use as the final input.
In this case, $M'$ has advantage $p_{u,u}-p_{x,u}\ge \delta/2$.

To prove the information complexity claim, observe that, until time $\T$, $M'$ just executes $M$, so we have
\begin{align}
    \sum_{t=1}^\T I(M_t; X_t\mid M_{t-1}, F) = \sum_{t=1}^\T I(M_t'; X_t\mid M_{t-1}', F).
\end{align}
$M'$ has one final step, but without loss of generality we can assume the final state of $M$ is a single bit (i.e., its answer) and thus we have, in case 1, that 
\begin{align}
    \sum_{t=1}^{\T+1} I(M_t'; X_t\mid M_{t-1}', F) 
        &=\sum_{t=1}^{\T} I(M_t; X_t\mid M_{t-1}) + I(M_{\T+1}'; X_{\T+1}\mid M_{\T}, F)\\
        &\le \sum_{t=1}^{\T} I(M_t; X_t\mid M_{t-1}, F) + 1
\end{align}
In case 2, $M'$ discards its final input and thus the same inequality holds.
\end{proof}

Combining this argument with Corollary~\ref{corollary:tasksubpopdist_hard}, we get the following lower bound.
\begin{corollary}\label{corollary:tasksubpoptestfixed_hard}
    Any algorithm $M$ solving \TaskSubpopTestFixed~on $\T$ examples with advantage at least $\delta$ satisifies
    \begin{align}
        \sum_{t=1}^\T I(M_t; X_t\mid M_{t-1}, F) \ge \frac{d}{160}\cdot \left(\frac{\delta}{\maxfixed+1} - \frac{\maxfixed}{d}\right)^4 - 1.
    \end{align}
\end{corollary}

\section{Proofs for Section~\ref{sec:core_lbd}}\label{app:core_lbd}

\begin{theorem}[Theorem~\ref{theorem:core_lbd} Restated]
    \corelbd
\end{theorem}

\begin{proof}
We first set up some notation.
Denote the sequence of arrivals' subpopulations~by $s\in[k]^\N$.
For sequence $s$, let $\ell_j$ be the number of arrivals from $j$, i.e., the number of time steps $t$ where $s_t=j$.
For the $a$-th arrival from subpopulation $j$, let $q_j(a)\in [\N]$ denote the time of that arrival.
For $a> \ell_j$, define $q_j(a)=\N+1$.
In this proof we sometimes abbreviate mutual information expressions of the form $I(A; B\mid C=c)$ to $I(A;B\mid c)$, leaving the random variable in the conditioning implicit.
We will also denote the length-$t$ prefix of a sequence $s$ via the notation $\slice{s}{}{t}$.
Similarly, $\slice{s}{t}{}$ denotes the suffix, all the terms from $t+1$ until the end.
Note that, contrary to programming language conventions, the start index is \textit{exclusive}.

We rewrite the composable information cost of $M$, identifying terms by the subpopulation of their input $X_i$ and their arrival number $a$:
\begin{align}
    \CI(M\mid S, F) 
        &= \E_s \sum_{i=1}^\N \left(\sum_{t=i}^{\N} I(M_t; X_i\mid M_{i-1}, F, s)\right) \\
        &= \E_s \sum_{j=1}^k \sum_{a=1}^{\ell_j} \left(\sum_{t=q_j(a)}^\N I(M_t; X_{q_j(a)}\mid M_{q_j(a)-1}, F, s)\right)
\end{align}
(Note that $L_j$ is a random variable that depends on $S$.)
We introduce an indicator random variable and write the arrival sum over $a=1,\ldots,\N$, since there are at most that many arrivals:
\begin{align}
    \CI(M\mid S,F) &= \E_s \sum_{j=1}^k \sum_{a=1}^{\N} \1{\ell_j\ge a} \sum_{t=q_j(a)}^\N I(M_t; X_{q_j(a)}\mid M_{q_j(a)-1}, F, s) \\
        &= \sum_{j,a}\E_s\left[ \1{\ell_j\ge a} \sum_{t=q_j(a)}^\N I(M_t; X_{q_j(a)}\mid M_{q_j(a)-1}, F, s) \right]\\
        &= \sum_{j,a}\Pr[L_j\ge a] \E_{s\mid L_j\ge a} \left[ \sum_{t=q_j(a)}^\N I(M_t; X_{q_j(a)}\mid M_{q_j(a)-1}, F, s) \right].
\end{align}
We now consider the expectation above and rewrite it as an expectation over 4 terms: 
    first the choice of $q_j(a)$, 
    then the choice\footnote{Because we condition on $L_j\ge a$ and consider the prefix up until the $a$-th arrival, this prefix has $j$ as its last element and contains $a-1$ other occurrences of $j$.} 
    of $\slice{s}{}{q_j(a)}$,
    then the choice of $q_j(a+1)$,
    and finally the choice of $\slice{s}{q_j(a)}{}$.
For brevity define $t_0=q_j(a)$ and $t_1=q_j(a+1)$.
Thus we have
\begin{align}
    \CI(M\mid S,F) &= \sum_{j,a} \Pr[L_j\ge a] \E_{\substack{t_0, \slice{s}{}{t_0}\\ \mid L_j\ge a}} \E_{\substack{t_1 \\\slice{s}{t_0}{}}}\left[\sum_{t=t_0}^{\N} I(M_t; X_{t_0}\mid M_{t_0-1}, F, s) \right] \\
        &= \sum_{j,a} \Pr[L_j\ge a]\E_{\substack{t_0, \slice{s}{}{t_0}\\ \mid L_j\ge a}} \E_{\substack{t_1 \\\slice{s}{t_0}{}}} \left[ Z_{\slice{s}{}{t_0}}^{j,a}\right],
\end{align}
using $Z_{\slice{s}{}{t_0}}^{j,a}$ as shorthand for the sum it replaces.
(Note that the inner expectation need not condition on $L_j\ge a$, since each entry in $s$ is drawn independently.)

We now analyze $\E_{t_1, \slice{s}{t_0}{}} \left[Z_{\slice{s}{}{t_0}}^{j,a}\right]$.
We are free to condition on the event $t_1=q_j(a+1)$ in the mutual information terms since it is a function of $\slice{s}{t_0}{}$; we do so and push the expectation over $\slice{s}{t_0}{}$ back inside the mutual information notation:
\begin{align}
    \E_{\substack{t_1,  \slice{s}{t_0}{}}} Z_{\slice{s}{}{t_0}}^{j,a} 
        &= \E_{t_1} \sum_{t=t_0}^{\N} I(M_t; X_{t_0}\mid M_{t_0-1},F,\slice{s}{}{t_0},\slice{S}{t_0}{},t_1).
\end{align}
At this point we focus on the terms between arrivals from subpopulation $j$, introducing an indicator random variable to discard the information terms corresponding to $M_{t_1}$ and beyond.
So we arrive at
\begin{align}
    \E_{\substack{t_1,  \slice{s}{t_0}{}}} Z_{\slice{s}{}{t_0}}^{j,a} 
        &\ge \E_{t_1} \sum_{t=t_0}^{\N} \1{t_1>t} I(M_t; X_{t_0}\mid M_{t_0-1},F, \slice{s}{}{t_0},\slice{S}{t_0}{},t_1) \\
        &=  \sum_{t=t_0}^{\N} \Pr[T_1> t] \E_{t_1\mid T_1 > t} I(M_t; X_{t_0}\mid M_{t_0-1},F, \slice{s}{}{t_0},\slice{S}{t_0}{},t_1) \\
        &=  \sum_{t=t_0}^{\N} \Pr[T_1> t] \cdot I(M_t; X_{t_0}\mid M_{t_0-1},F, \slice{s}{}{t_0},\slice{S}{t_0}{},T_1,T_1> t),
\end{align}
pushing the expectation over $T_1$ into the information terms and adding the condition that $T_1> t$.

We now execute the fundamental operations in the proof.
The first is that the mutual information terms do not change when conditioning on $T_1=t+1$ in place of $T_1>t$, since these events depend on the sequence \emph{after} time $t$.
We align the two expressions to highlight the switch:
\begin{align}
    I(M_t; X_{t_0}&\mid M_{t_0-1},F, \slice{s}{}{t_0},\slice{S}{t_0}{},\textcolor{blue}{T_1,T_1> t} ) \nonumber \\
        = I(M_t; X_{t_0}&\mid M_{t_0-1},F, \slice{s}{}{t_0},\slice{S}{t_0}{},\textcolor{orange}{T_1=t+1} ).
\end{align}
The second fundamental operation is to observe that, for any fixed $t_0$, $T_1-t_0$ is a truncated geometric random variable, since at each time step we observe an example from subpopulation $j$ with probability $\frac{1}{k}$, under the restriction that $T_1\le \N$.
Thus, for any $t_0 \le t< \N$, 
\begin{align}
    \Pr[T_1 > t] = k\cdot \Pr[T_1 = t + 1].
\end{align}
When $t=\N$ we have $\Pr[T_1>\N]=\Pr[T_1=\N+1]$.
Thus we have
\begin{align}
    \E_{\substack{t_1,  \slice{s}{t_0}{}}} Z_{\slice{s}{}{t_0}}^{j,a} 
        &\ge \sum_{t=t_0}^{\N} \Pr[T_1>t]\cdot I(M_t; X_{t_0}\mid M_{t_0-1},F, \slice{s}{}{t_0},\slice{S}{t_0}{},T_1,T_1=t+1) \\
        &= \sum_{t=t_0}^{\N-1} k\cdot \Pr[T_1=t+1]\cdot I(M_t; X_{t_0}\mid M_{t_0-1},F, \slice{s}{}{t_0},\slice{S}{t_0}{},T_1=t+1) \\
        &\quad +\Pr[T_1=\N+1]\cdot I(M_{\N}; X_{t_0}\mid M_{t_0-1},F, \slice{s}{}{t_0},\slice{S}{t_0}{},T_1=\N+1) \\
        &= \sum_{t=t_0}^{\N-1} k\cdot f(t) + f(\N). \label{eq:shorthand_with_fs}
\end{align}
defining the shorthand $f(t)$:
\begin{align}
    f(t) \defeq \Pr[T_1=t+1]\cdot I(M_t; X_{t_0}\mid M_{t_0-1},F, \slice{s}{}{t_0},\slice{S}{t_0}{},T_1=t+1).
\end{align}
To proceed, we need the following observation about this function $f(\cdot)$.
\begin{claim}\label{claim:f_decreases}
    $f(t)$ is nonincreasing in $t$. 
\end{claim}
\begin{proof}
    $f(t)$ is a product of two terms, both of which are nonincreasing in $t$.
    This holds for the mutual information term by Proposition~\ref{proposition:streaming_DPI} (the DPI for streaming).
    It holds for the probability term because $T_1-t_0$ is truncated geometric.
\end{proof}

Continuing from Equation~\eqref{eq:shorthand_with_fs}, we have the following lower bound:
\begin{align}
    \E_{\substack{t_1,   \slice{s}{t_0}{}}} Z_{\slice{s}{}{t_0}}^{j,a}  
        \ge k\sum_{t=t_0}^{\N-1} f(t) + f(\N)
        \ge \frac{k}{2} \cdot \1{t_0 \neq N} \sum_{t=t_0}^{\N} f(t).
\end{align}
To see this, first note that we can discard $f(\N)$, as it is nonnegative.
If $t_0=\N$ the lower bound is vacuous, so assume otherwise and use Claim~\ref{claim:f_decreases} to write 
\begin{align*}
    k\cdot f(\N-1) = \frac{k}{2}f(\N-1) + \frac{k}{2}f(\N-1) 
        \ge \frac{k}{2}f(\N-1) + \frac{k}{2}f(\N). 
\end{align*}
Hit the earlier terms in the sum with a factor of $\frac{1}{2}$.

We can rewrite the sum $\sum_t f(t)$ as an expectation over $T_1=Q_j(a+1)$, recalling that $T_1$ is a function of $\slice{S}{t_0}{}$ and we need not condition on both.
\begin{align}
    \sum_{t=t_0}^\N f(t) &= \sum_{t=t_0}^\N \Pr[T_1=t+1]\cdot I(M_t; X_{t_0}\mid M_{t_0-1},F, \slice{s}{}{t_0},\slice{S}{t_0}{},T_1=t+1) \\
    &= I(M_{T_1-1}; X_{t_0}\mid M_{t_0-1}, F, \slice{s}{}{t_0}, \slice{S}{t_0}{}) \\
    &= I(M_{Q_j(a+1)-1}; X_{q_j(a)}\mid M_{q_j(a)-1}, F, \slice{s}{}{q_j(a)}, \slice{S}{q_j(a)}{}).
\end{align}

Thus we reach the following lower bound on $\CI(M\mid S, F)$, pulling the expectation over $s$ back out to the front.
\begin{align}
    \CI(M \mid S, F) 
    &= \sum_{j,a} \Pr[L_j\ge a]\E_{\substack{t_0, \slice{s}{}{t_0}\\ \mid L_j\ge a}} \E_{\substack{t_1 \\\slice{s}{t_0}{}}} \left[ Z_{\slice{s}{}{t_0}}^{j,a}\right] \\
    &\ge \sum_{j,a} \Pr[L_j\ge a] \E_{\substack{t_0, \slice{s}{}{t_0}\\ \mid L_j\ge a}} \left[ \frac{k}{2} \cdot \1{t_0\neq \N} \sum_{t=t_0}^{\N} f(t) \right] \\
    &= \frac{k}{2} \E_{s} \sum_{j=1}^k \sum_{a=1}^{\ell_j} \1{q_j(a)\neq \N} I(M_{q_j(a+1)-1}; X_{q_j(a)}\mid M_{q_j(a)-1}, s, F),
\end{align}
For any sequence $s$, the indicator random variable $\1{q_j(a)\neq \N}$ will be zero exactly once, corresponding to the final arrival. 
Since all of these information terms are bounded above by $H(X_t)\le d$, we have
\begin{align}
     \CI(M\mid S, F) &\ge \frac{k}{2} \E_{s} \sum_{j=1}^k \sum_{a=1}^{\ell_j} I(M_{q_j(a+1)-1}; X_{q_j(a)}\mid M_{q_j(a)-1}, s, F) -  \frac{kd}{2}.
     \label{eq:all_infos}
\end{align}

\begin{algorithm2e}\label{alg:random_to_subpop}
    \SetAlgoLined
    \SetKwInOut{Input}{input}
    
    \Input{structured stream $x_1,\ldots, x_{\N}\in\{0,1\}^d$; test example $x_{\mathrm{test}}$; algorithm $M$;  index $j\in [k]$; sequence $s\in [k]^\N$; parameters $d, k, \maxfixed, \N$}
    \BlankLine
    For $i\in [k]\setminus \{j\}$, sample subpopulation parameters $F_i$\;
    \For{$\ell=1,\ldots, \N$}{
        \eIf{$s_\ell = j$}{
            Execute $M(x_t, j)$\;
            $t\gets t+ 1$\;
        }{
            Draw $z_\ell \sim F_{s_\ell}$\tcc*{generate synthetic input}
            Execute $M(z_\ell, s_\ell)$\;
        }
    }
    Receive trained classifier $M_{\N}$\;
    \KwRet $M_{\N}(x_{\mathrm{test}})$
    \caption{$M'$ for \TaskSubpopTestFixed}
\end{algorithm2e}

For any subpopulation $j$ and sequence $s$ (which fixes $\ell_j$), Algorithm~\ref{alg:random_to_subpop} defines an algorithm solving \TaskSubpopTestFixed~on $\ell_j$ examples.
Denote the algorithm $M_{j,s}$ and let $\delta_{j,s}$ be its advantage.
By construction, $\E_{j,s}[\delta_{j,s}]=\delta$, where $\delta$ is the advantage of $M$.
By Corollary~\ref{corollary:tasksubpoptestfixed_hard}, we have a lower bound on the composable information cost of $M_{j,s}$:
\begin{align}
    \sum_{a=1}^{\ell_j} I(M_{q_j(a+1)-1}; X_{q_j(a)}\mid M_{q_j(a)-1}, s, F) \ge \frac{d}{160}\cdot \left(\frac{\delta_{j,s}}{\maxfixed+1} - \frac{\maxfixed}{d}\right)^4 - 1.
\end{align}
Note that this expression is convex in $\delta_{j,s}$, so we plug the lower bound into Equation~\eqref{eq:all_infos}, rewrite the sum over $j$ as an expectation, and apply Jensen's inequality:
\begin{align}
    \CI(M\mid S, F) 
        &\ge \frac{k}{2} \E_{s} \sum_{j=1}^k \left(\frac{d}{160}\cdot \left(\frac{\delta_{j,s}}{\maxfixed+1} - \frac{\maxfixed}{d}\right)^4 - 1 \right) -  \frac{kd}{2} \\
        &=\frac{k^2}{2} \E_{s,j} \left(\frac{d}{160}\cdot \left(\frac{\delta_{j,s}}{\maxfixed+1} - \frac{\maxfixed}{d}\right)^4 - 1 \right) -  \frac{kd}{2} \\
        &\ge \frac{k^2d}{320} \cdot \left(\frac{\E_{s,j}[\delta_{j,s}]}{\maxfixed+1} - \frac{\maxfixed}{d}\right)^4 - \frac{k^2}{2} -  \frac{kd}{2}.
\end{align}
Since $\E_{j,s}[\delta_{j,s}]=\delta$, we are done.
\end{proof}

\section{An Upper Bound for the Core Problem}\label{app:core_ubd}

In this section, we present a time-and-space efficient algorithm for \TaskCore~and sketch its analysis.
On streams of length $\N$, it uses $\tilde{O}\left(\frac{k^2d}{N\maxfixed}\right)$ space.
When $\N\gtrsim k \log k \log d,$ the algorithm will have constant error.
Recall that, when $\maxfixed = o(d^{1/4})$, Theorem~\ref{theorem:core_lbd} gives a space lower bound of $\Omega\left(\frac{k^2d}{N\maxfixed^4}\right)$.
Importantly, we do \textit{not} prove that this algorithm is an agnostic learner that works for any distribution: we only prove that it has a constant advantage (see Definition~\ref{def:alltasks}).

We present $M$ in Algorithm~\ref{alg:ubd_for_coretask}.
It proceeds over $\tau$ epochs, within each epoch attending to only a subset of the $k$ subpopulations.\footnote{The space usage can be lowered further by working with only a subset of indices. 
For simplicity, we ignore this regime.}

\begin{algorithm2e}\label{alg:ubd_for_coretask}
    \SetAlgoLined
    \SetKwInOut{Input}{input}
    
    \Input{stream $(x_1,j_1),\ldots,(x_\N,j_\N)$; test example $x_{\mathrm{test}}$; parameters $d, k, \maxfixed, \N, c$}
    \BlankLine
    Set $\N_0\gets c k\log k\log d$, $\tau \gets \lfloor \N/\N_0\rfloor$, $d'\gets \lceil d/\maxfixed\rfloor$\;
    Create partition $S_1,\ldots,S_\tau \subseteq [k]$\tcc*{each of size $k/\tau$}
    \For{$t=1,\ldots, \tau$}{
        For all $\ell\in S_t$, set $x_\ell^{\mathrm{ref}} \gets \texttt{NULL}$\tcc*{Leave uninitialized}
        For all $\ell\in S_t$, set $C_\ell\gets 1^{d'}$\tcc*{Candidate indices}
        \For{$i=1,\ldots, \N_0$}{
            Receive next example $(x,j)$\;
            \If{$j\in S_t$}{
                \eIf{$x_j^{\mathrm{ref}}=\texttt{NULL}$}{
                    Set $x_j^{\mathrm{ref}}\gets x[1:d']$\tcc*{Set reference string}
                }{
                    Set $C_j \gets C_j \land (x = x_j^{\mathrm{ref}})$\tcc*{bitwise AND, EQUAL}
                }
            }
        }
        \For{$\ell\in S_t$}{
            Set $B_\ell\gets 
                \{(a,x_\ell^{\mathrm{ref}}[a]): a\in [d'], C_\ell[a]=1\}$
                \tcc*{Fixed bits, values}
            \If{$|B_\ell|\ge \maxfixed$}{
                \KwRet \texttt{FAIL}\;    
            }
        }
    }
    Receive $(x_{\mathrm{test}}, j)$\;
    \For{$(a, b)\in B_j$}{
        \If{$x_{\mathrm{test}}[a]\neq b$}{
            \KwRet \unif\;
        }
    }
    \KwRet \struct\;
    \caption{$M'$ for \TaskSubpopTestFixed}
\end{algorithm2e}

\paragraph{Space Analysis}
Storing the list of partitions $S_1,\ldots, S_\tau$ requires $O(k \log k)$ bits.
During any epoch, $M$ works with $O(k/\tau)$ subpopulations and stores a reference string ($d'$ bits) and tracks candidate indices via another array of $d'$ bits.
It also tracks the $B_j$ lists after each epoch: there are at most $k$ of these and they require at most $O(\maxfixed \log d)$ bits to specify.
Thus $M$ requires space
\begin{align*}
    O\left(k \log k + \frac{k}{\tau}\cdot d' + k\maxfixed \log d\right) 
    = O\left(\frac{k^2 d \log k \log d}{\maxfixed N} \right),
\end{align*}
which matches our lower bound up to a factor of $\frac{1}{\maxfixed^3} \log k \log d$.

\paragraph{Error Analysis}
We show that this algorithm has a constant advantage.
Setting $\N_0 = c k \log k \log d$ for sufficiently large constant $c$ ensures that, with high constant probability, within each of the $\tau $ epochs all $k$ subpopulations receive at least $c' \log d$ examples for some constant $c'$ (via the $m$-copy coupon collector problem).
When $M$ sees $c'\log d$ examples from a subpopulation $j$ (during an epoch $t$ in which $j\in S_t$), it will with high probability discard all unfixed features from the set $[d']$, and be left with a list $B_j$ that is either (i) empty or (ii) contains only indices of fixed features.
Since there are no more than $\maxfixed$ fixed features, this implies that with high probability $M$ does not return \texttt{FAIL}.

We have established that, with high probability, $M$ will identify all of the fixed features in the first $d'=d/\maxfixed$ indices. 
This setting of $d'$ ensures that, with constant probability over the choice of $r\in \{0,\ldots, \maxfixed\}$ and $\{j_1,\ldots, j_r\}\in [d]$, at least one fixed feature will land in $[d']$.
When $M$ has correctly identified all the fixed indices in the first $d'$ indices (and there is at least one fixed feature), it will always output \struct~on structured inputs and will output \struct~w.p. at most $\frac{1}{2}$ in the uniform case.

\section{Agnostic Learning of Direct Sums of $k$ Dictators}\label{app:agnostic_dictators}

We show how to time-and-space efficiently agnostically learn the Direct Sum of $k$ Dictators class described in the introduction.
Recall the definition: \kdictators~Note that this class has size $(d')^k\le d^k$, so (by the standard uniform convergence argument for finite hypothesis classes, i.e., a Chernoff bound and union bound) $O(k \log d)$ samples suffice to guarantee that with constant probability ERM returns a classifier within constant accuracy of the best possible.
Let $\kappa = k \log d$.

We first show how to implement ERM using $O(\kappa)$ samples and $O(d\kappa)$ space and time.
By definition, the probability that a classifier $h_{i_1,...,i_k}$ misclassifies a point $(j, x)\in \cX$ is the probability that the label $y$ differs from the bit $x_{i_j}$.
To track the empirical error, then, it suffices to store a matrix $A\in \mathbb{R}^{k\times d}$, initialized to all zeros, and update it upon receiving a labeled example $((j,x),y)\in \cX\times \{0,1\}$ in the following way:
\begin{align*}
    \forall i \in[d],\quad A_{j,i} 
        = \begin{cases} 
            A_{j,i} + 1 & \text{if $x_i \neq y$} \\
            A_{j,i} & \text{otherwise}
        \end{cases}.
\end{align*}
After $\kappa$ examples, we select the classifier $h_{i_1,...,i_k}$ with the smallest empirical error: for each row $j\in [k]$ we select $i_j^* = \mathrm{argmin}_{i} A_{j,i}$, the index that minimizes the error.

This correctly implements ERM, so it is an agnostic learner.
For each bit of our input we execute $O(1)$ operations, so the time used is $O(d\kappa)$.
The space usage is just the matrix $A$, which has $kd$ integer entries, each of which is between 0 and $\kappa$, so the algorithm uses $O(kd \log \kappa)$ space (this is $O(d\kappa)$ as long as $\log d\le k^{O(1)}$).

This algorithm naturally extends to longer streams of length $N = O(\tau \kappa)$ for $\tau\ge \Omega(1)$.
We run a similar procedure over $\tau$ epochs, in each epoch receiving $\kappa$ examples and working with a $\frac{1}{\tau}$ fraction of the rows of $A$, i.e., a subset of the subpopulations.
At the end of the epoch, we store the minimum-error indices $i_j^*$ for the rows we consider.
At the end of the stream we assemble these indices to pick a single classifier.
This algorithm implements ERM, runs in time $O(\tau \kappa d)$, and uses space $O\left(\frac{d\kappa\log \kappa}{\tau}\right)= O\left( \frac{d\kappa^2 \log \kappa}{N}\right)$, matching our lower bound up to logarithmic factors.

\section{Reductions from Core Problem to Agnostic Learning}\label{agnostic2core}

In this section we show how agnostic learning algorithms for several natural functions classes can be turned into constant-advantage learning algorithms for \TaskCore.
Throughout, ``agnostic learning'' implies learning to sufficiently small constant accuracy with sufficiently high constant probability.

For each hypothesis class below, we restate the definition and then show (i) how to turn a labeled example from \TaskCore~into a labeled example for the given hypothesis class and (ii) how to use an agnostically learned hypothesis $h^*$ and a test example from \TaskCore~to output an answer (either \struct or \unif, for ``uniform'' and ``structured'') with constant advantage.
None of these reductions require additional space or examples.

\paragraph{Direct Sums of $k$ Dictators}
\kdictators

Let $\cH_{DS}$ denote the hypothesis class described above.
We reduce from \TaskCore~with $\maxfixed=1$ and data dimension $d'$, creating an algorithm $M$ for \TaskCore~that uses an agnostic learning algorithm for the $\cH_{DS}$.
Given an example $(x,j)$ from \TaskCore, $M$ constructs a labeled example $((j, x'), y)$ by drawing $y\in\{0,1\}$ randomly and setting 
$x'\gets (y\cdot 1^d)\oplus x$, where $\oplus$ denotes bitwise XOR.
Given a learned hypothesis $h^*$ and test example $(j,x_{\mathrm{test}})$, $M$ constructs $((j,x_{\mathrm{test}}'),y)$ in the same manner and outputs \struct if $h^*((j',x_{\mathrm{test}}'))=y$ and \unif otherwise.

The XOR operation ensures that, when a subpopulation $j$ has a feature $i\in [d']$ fixed to 0, the label $y$ of $(j,x)$ is 1 iff $x_{i}=1$.
For these subpopulations, then, there is a dictator that labels them exactly.
For the other subpopulations (those with a feature fixed to 1, or with no fixed features), there is no dictator that labels examples with accuracy better than $\frac{1}{2}$.
With $\maxfixed=1$, subpopulations have a 1-in-4 chance of getting a fixed feature with value 0; this choice is independent across subpopulations, so in expectation (over the choice of distribution) the best hypothesis in $\cH_{DS}$ will have accuracy $\frac{1}{2}\left(1 - \frac{1}{4}\right) + 1\cdot \frac{1}{4} = \frac{1}{2} + \frac{1}{8}$. 
With high probability an agnostic learning algorithm will return a hypothesis $h^*$ with accuracy within a small constant of that, so in expectation $h^*$ will have accuracy at least $\frac{1}{2}+c'$ for some positive constant $c'$.

We now show that the algorithm $M$ for \TaskCore~described above has constant advantage.
Recall the definition of advantage:
\coreadvantage
Here $\mixture$ denotes the structured distribution and $\cU$ denotes the uniform distribution over $[k]\times \{0,1\}^d$.
Since $M$ outputs $\struct$ exactly when $h^*$ is correct, the first probability is at least $\frac{1}{2}+c'$ in expectation.
When the input is drawn from $\cU$, the label $y$ is uniform and independent of the pair $(j,x')$, so $h^*$ is correct with probability exactly $\frac{1}{2}$.

We have established that an agnostic learning algorithm for $\cH_{DS}$ yields a constant-advantage learner for \TaskCore with the same space and sample efficiency, so any agnostic learner for $\cH_{DS}$ requires $\Omega\left(kd'\cdot \frac{k}{\N}\right)$ bits of memory.
When $d=\omega(\log k)$, this is $\Omega\left(kd\cdot\frac{k}{\N}\right)$.

\paragraph{Sparse Linear Classifiers over the Degree-2 Polynomial Features}
\sparsekernel

Let $\cH_{SL}$ denote the hypothesis class described above.
We reduce from \TaskCore~by reducing from the Direct Sums of $k$ Dictators class: any algorithm for agnostically learning $\cH_{SL}$ on $d$ dimensions can be used to agnostically learn $\cH_{DS}$.
Let $f:[k]\to \{0,1\}^k$ represent one-hot encoding and let $d'=d-k$.
Given a labeled example $((j,x),y)\in [k]\times \bit{d'}\times \bit{}$ for the Direct Sum of $k$ Dictators class, construct a labeled example $(f(j)\circ x, y)$, where $\circ$ denotes concatenation.
Note that this construction requires $d\ge k$.

For any function $h_{i_1,...,i_k}$ in the Direct Sums of $k$ Dictators class, there is a function $h(z) = \mathrm{sign}(\langle w, k(z)\rangle)$ with $k$-sparse $w\in \{0,1\}^{\binom{d}{\le 2}}$ such that for all $(j, x)$ we have $h_{i_1,...,i_k}(j,x)=h(f(j)\circ x)$.
Explicitly, indexing into $w$ with pairs $(i,j)$, we construct
\begin{align*}
    w_{(i,j)} = \begin{cases}
                    1 & \text{if $j\le k$ and $i = i_j + k$} \\
                    0 & \text{otherwise}
                \end{cases}.
\end{align*}
Thus, under this encoding, $\cH_{DS}\subseteq \cH_{SL}$ and any agnostic learning algorithm for $\cH_{SL}$ is also one for $\cH_{DS}$, since for a classifier $h^*$ and constant $c$ we have
\begin{align*}
    \err(h^*) \le \min_{h\in \cH_{SL}} \err(h) + c \le \min_{h\in \cH_{DS}} \err(h) + c.
\end{align*}
This implies that agnostically learning $\cH_{SL}$ for $d\ge k$ requires $\Omega\left(k (d-k)\cdot \frac{k}{\N}\right)$ bits of memory.
For $d\ge 2k$, this is $\Omega\left(k d\cdot \frac{k}{\N}\right)$.

\paragraph{$k$-term 2-DNFs}
\ktermdnfs

Let $\cH_{DNF}$ denote the hypothesis class described above.
We use the same reduction as we used in $\cH_{SL}$: any algorithm for agnostically learning $k$-term 2-DNFs can be used to learn $\cH_{DS}$.
We use the same encoding: given a labeled example $((j,x),y)$, construct a labeled example $(f(j)\circ x, y)$.
For every function $h_{i_1,\ldots, i_k}\in \cH_{DS}$, there is an equivalent function $h\in \cH_{DNF}$, namely:
\begin{align*}
    h(x) = \bigvee_{j=1}^k \left( x_{j} \wedge x_{k+i_j}\right).
\end{align*}
Thus the lower bound for $\cH_{SL}$ also applies here, and for any $d\ge 2k$ agnostically learning $\cH_{DNF}$ requires $\Omega\left(kd \cdot \frac{k}{\N}\right)$ bits of memory.

\paragraph{Multiclass Sparse Linear Classifiers}
\multiclasslinear

We reduce from \TaskCore~with $\maxfixed=O(\log k)$.
Given an example $(j,x)$ from \TaskCore, we construct a labeled example $(x,j)$, with the subpopulation identifier as the label.
Given a hypothesis $h^*$ and test example $(j,x_{\mathrm{test}})$, we output \struct~if $h^*(x_{\mathrm{test}})=j$ and \unif~otherwise.

Since with $\maxfixed=O(\log k)$ the optimal linear multiclass classifier of this form has at most low constant error, this output will be correct with high constant probability and we have a space lower bound of $\Omega\left(\frac{k^2 d}{\N\maxfixed^4}\right) = \Omega\left(\frac{k^2 d}{\N \log^4 k}\right)$.

\paragraph{Real-Valued Regression}
\realregression

We reduce from \TaskCore~with $\maxfixed=O(\log k)$.
Note that we can express the set of target values as $\bigl\{0,\frac{1}{k-1},\frac{2}{k-1},\ldots, 1\bigr\} = \bigr\{\frac{j-1}{k-1}\bigr\}_{j=1}^k$.
We modify the previous reduction: given an example $(j,x)$ from \TaskCore, we construct a labeled example $\left(x,(j-1)/(k-1)\right)$.
Given a hypothesis $h^*$ and test example $(j,x_{\mathrm{test}})$, we compute $\beta = \left|h^*(x_{\mathrm{test}})- \frac{j-1}{k-1}\right|$ and flip a coin that comes up heads with probability $\beta$.
If the coins is heads output \unif~and if tails output \struct.

Let $E$ be the event that the coin comes up heads.
On a structured input, the probability the reduction described above yields the incorrect answer is
\begin{align*}
        \Pr[E] = \E_j\sum_{x} \Pr[X_{\mathrm{test}}=x] \cdot \left|h^*(x_{\mathrm{test}})- (j-1)/(k-1)\right|,
\end{align*}
exactly the expected error of the regression model on the structured distribution.

We now show that, with $\maxfixed=O(\log k)$, the optimal regression model of this form has at most small constant error, say $\frac{1}{10}$.
Each hidden node has a wire to the output node with a weight of the form $\frac{j-1}{k-1}$; call this  ``hidden node $j$.''
The activation function for this node is ReLU, with a fixed offset (or \textit{bias}) $b_j$: $f(x) = \max\braces{ 0, x-b_j }$.
Call a subpopulation \textit{good} if it has at least $2\log k$ indices fixed to 1.
We now construct a specific neural network.
If subpopulation $j$ is not good, we set to 0 the weights of all wires incoming to hidden node $j$.
If subpopulation $j$ is good, we (i) take $2\log k$ wires running from indices fixed to 1 to hidden node $j$ and set their weights to 1, and (ii) set $b_j=(2\log k)-1$ as the bias.
We claim that, with high probability over the choice of distribution, this construction has low constant error (in particular, with high probability over the choice of example it produces the exact correct label).
To see this, first note that for $\maxfixed \gg 2\log k$, with high probability a large constant fraction of the subpopulations are good.
If this occurs, the neural network will have low constant error.
Let random variable $Z_j$ be the input to hidden node $j$ on a random input $X$ (drawn from the structured distribution, as in the reduction above).
The label for $X$ is $\frac{j-1}{k-1}$; the neural network outputs this value exactly if (i) $j$ is a good subpopulation and (ii) for every other subpopulation $\ell \neq j$, $Z_\ell < 2\log k$.
Condition (i) happens with high constant probability by assumption.
Condition (ii) happens with probability at most $\frac{1}{k}$, since $Z_\ell=2\log k$ only when all the wires leading to hidden node $\ell$ receive input 1, which happens with probability exactly $\frac{1}{k^2}$.
A union bound over the (at most) $k-1$ other good subpopulations concludes the argument.

On uniform inputs, the output $h^*(X_{\mathrm{test}})$ generated in the reduction has some distribution that is independent of $j$.
Letting $\mu = \E_{x\sim \cU}[h(x)]$, we have by Jensen's inequality that
\begin{align*}
    \Pr[E] &= \E_j\E_{x_{\mathrm{test}}\sim \cU}\left[|h^*(x_{\mathrm{test}}) - (j-1)/(k-1) |\right] \\
        &\ge  \E_j\left[|\mu - (j-1)/(k-1) |\right].
\end{align*}
This is at least $\frac{1}{8} - O(1/k)$, with the lower-order term coming from the fact that $j$ is discrete.
Informally, with probability $\frac{1}{2}$ we must have $|\mu -j/k|\ge \frac{1}{4}$.

Because $\Pr[E]$ has constant separation between the two cases, we have a learner for \TaskCore~with constant advantage and a space lower bound of $\Omega\left(\frac{k^2 d}{\N \log^4 k}\right)$ for this real-valued regression task.

\tableofcontents 

\end{document}